\RequirePackage[svgnames]{xcolor}

\documentclass[11pt,letterpaper]{mystyle}
\usepackage{amsmath}
\usepackage{amssymb}
\usepackage[all]{hypcap}
\usepackage[svgnames]{xcolor}
\usepackage[comma,authoryear,compress]{natbib}
\usepackage{hyperref}[citecolor=lightblue]

\hypersetup{
    colorlinks = true,
    citecolor = {violet},
}

\usepackage{algorithm}
\usepackage{algorithmicx}
\usepackage{algpseudocode}
\usepackage{microtype}
\usepackage{graphicx}
\expandafter\def\csname ver@subfig.sty\endcsname{}
\usepackage{booktabs} %
\usepackage{float}
\usepackage{bigstrut}

\usepackage{mathtools}
\usepackage{amsthm}
\usepackage{mathrsfs}
\usepackage{nicefrac}
\usepackage{dsfont}
\usepackage{enumitem}
\usepackage{subcaption}
\usepackage{graphicx,subfig}
\usepackage{cleveref}
\usepackage{bxcoloremoji}
\usepackage{datetime2}
\usepackage{float}
\usepackage{physics}

\usepackage[utf8]{inputenc} %
\usepackage[T1]{fontenc}    %
\usepackage{hyperref}       %
\usepackage{url}            %
\usepackage{booktabs}       %
\usepackage{amsfonts}       %
\usepackage{nicefrac}       %
\usepackage{microtype}      %
\usepackage{graphicx}
\usepackage{subcaption} 
\usepackage{wrapfig}
\usepackage{lipsum}
\usepackage{enumitem}
\usepackage{stackengine}
\usepackage[font=small,labelfont=bf]{caption}
\usepackage{color}
\usepackage{adjustbox}
\usepackage{multirow, multicol}

\usepackage{rotating}
\usepackage{makecell}

\usepackage{tcolorbox}
\tcbuselibrary{listings,skins,breakable}
\usepackage{listings}
\usepackage{subcaption}

\newcommand{\x}{\mathbf{x}}

\DeclareMathOperator*{\E}{\mathbb{E}}

\definecolor{blanchedalmond}{rgb}{1.0, 0.92, 0.8}
\definecolor{carmine}{rgb}{0.59, 0.0, 0.09}
\definecolor{lightblue}{rgb}{0.22,0.45,0.70}%

\newtheorem{theorem}{Theorem}[section]

\renewcommand{\mathbf}{\boldsymbol}

\makeatletter
\def\Ddots{\mathinner{\mkern1mu\raise\p@
\vbox{\kern7\p@\hbox{.}}\mkern2mu
\raise4\p@\hbox{.}\mkern2mu\raise7\p@\hbox{.}\mkern1mu}}
\makeatother

\definecolor{amaranth}{rgb}{0.9, 0.17, 0.31}
\definecolor{antiquebrass}{rgb}{0.8, 0.58, 0.46}
\definecolor{antiquefuchsia}{rgb}{0.57, 0.36, 0.51}
\definecolor{chromeyellow}{rgb}{0.31, 0.47, 0.26}

\newtcolorbox{AIbox}[2][]{aibox,title=#2,#1}
\definecolor{lightgreen}{rgb}{0.22,0.70,0.30}%
\definecolor{Gray}{gray}{0.95}
\definecolor{Cornsilk}{rgb}{1.0, 0.97, 0.86}

\definecolor{lightblue}{HTML}{0064E0}
\definecolor{fg}{HTML}{1C2B33}
\definecolor{bg}{HTML}{F1F4F7}

\newcommand{\name}{\textsc{Rubric-ARROW}}

\providecommand{\Prob}{\mathbb{P}}
\providecommand{\norm}[1]{\left\lVert #1 \right\rVert}

\usepackage{fvextra}
\definecolor{promptpink}{RGB}{205,120,154}
\definecolor{promptbg}{RGB}{250,242,246}

\usepackage{amsmath}

\usepackage[all]{hypcap}
\title{\LARGE  {\name}: Alternating Pointwise Rubric Reward Modeling for LLM Post-training in Non-verifiable Domains}

\runningtitle{Alternating Pointwise Rubric Reward Modeling}

\author{
Haoxiang Jiang\textsuperscript{1} \quad
Zihan Dong\textsuperscript{2} \quad
Tianci Liu\textsuperscript{3} \quad
Wanying Wang\textsuperscript{4} \quad
Ran Xu\textsuperscript{5} \quad
Tony Yu\textsuperscript{6} \quad
Linjun Zhang\textsuperscript{2} \quad
Haoyu Wang\textsuperscript{1} \\
\textsuperscript{1}University at Albany \quad
\textsuperscript{2}Rutgers University \quad
\textsuperscript{3}Purdue University \quad
\textsuperscript{4}Independent Researcher \quad
\textsuperscript{5}Emory University \quad
\textsuperscript{6}Georgia Institute of Technology
}

\begin{document}

\begin{abstract}
Pointwise reward modeling offers critical signals for LLM post-training, yet struggles with absolute scoring in subjective, non-verifiable settings. Rubric-based methods address this by decomposing evaluation into explicit criteria, but existing approaches typically depend on frontier LLMs and suffer from ties caused by hard Boolean aggregation. We present \name, an alternating framework that jointly trains a rubric generator and a rubric-conditioned judge, with its RL stage using only pairwise preference data. Our method couples a probability-based scoring rule that reduces ties with phase-specific preference-based rewards and an alternating GRPO scheme that together train the pointwise evaluator. Extensive experiments show that \name\ achieves competitive reward-modeling accuracy and yields consistent gains for downstream policy post-training.
\vspace{2mm}


\vspace{5mm}

\coloremojicode{1F4C5} \textbf{Date}: \today

\coloremojicode{1F917} \textbf{Model Weights \& Checkpoints}: \href{https://huggingface.co/collections/OpenRubrics/rubricarrow}{https://huggingface.co/collections/OpenRubrics/rubricarrow}

\coloremojicode{1F4DA} \textbf{Datasets}: \href{https://huggingface.co/datasets/OpenRubrics/RubricARROW-Judge-SFT}{https://huggingface.co/datasets/OpenRubrics/RubricARROW-Judge-SFT}

\coloremojicode{1F4E7} \textbf{Contact}: \href{mailto:hjiang2@albany.edu}{hjiang2@albany.edu};
\href{mailto:hwang28@albany.edu}{hwang28@albany.edu}

\end{abstract}

\maketitle
\vspace{3mm}
\section{Introduction}

Reward modeling is a core component of LLM post-training~\citep{stiennon2020learning,ouyang2022training,bai2022training}.
In particular, pointwise reward modeling is attractive, as they assign each response an independent scalar score~\citep{wang2024interpretable}, enabling efficient evaluation and direct policy optimization.
Yet in non-verifiable tasks such as open-ended instruction following and chat, a single holistic score is often unreliable~\citep{ying2025beyond,hashemi2024llm,lee2025checkeval,song2024finesure}. This has motivated rubric-based reward modeling, which decomposes evaluation into explicit criteria and aggregates criterion-level judgments into a pointwise reward score.

However, rubric-based reward modeling introduces a new challenge: a reliable evaluator requires not only high-quality criteria, but also a judge that can apply them consistently. A common solution is to prompt frontier LLMs to generate rubrics and judgments~\citep{viswanathan2026checklists,chen2025judgelrm,gunjal2025rubrics,liu2023g}. While effective, this paradigm largely outsources evaluation to external models, making reward estimation expensive, difficult to control, and hard to deploy as a standalone reward model.

Recent learning-based approaches take steps toward trainable rubric evaluators, but they remain limited in three important ways. First, several methods train either the rubric generator or the judge, without jointly optimizing the two components altogether~\citep{liu2025openrubrics,rezaei2025online}. Second, many approaches rely on frontier LLMs as the judge to provide supervision for reward modeling~\citep{zhou2025breaking,li2026rubrichub,he2025advancedif}, preserving the dependence on external evaluators. Third, the supervision most commonly available for alignment is pairwise preference data, which does not directly specify either absolute pointwise scores or criterion-level labels. This mismatch makes it unclear how to train a rubric-conditioned pointwise reward model from preferences alone. Compounding this issue, hard Boolean aggregation over rubric criteria often collapses distinct responses into tied scores, weakening the discriminative signal needed for both reward modeling and downstream policy optimization~\citep{zhang2025chasing}.

Motivated by challenges above, we introduce \name\footnote{Short for \textbf{a}lternating pointwise \textbf{r}ubric \textbf{r}eward m\textbf{o}deling frame\textbf{w}ork.}, an alternating training scheme for  \emph{pointwise} rubric reward models from pairwise preferences. \name{} treats rubric-based evaluation as a two-module learning problem: a \emph{rubric generator} proposes criteria, and a \emph{rubric-conditioned judge} assigns criterion-level scores that are aggregated into a pointwise reward. Crucially, both modules are optimized with reinforcement learning using only pairwise preference data in the RL stage, avoiding reliance on frontier-LLM-labeled judge supervision. To obtain informative rewards, we replace hard Boolean aggregation with a probability-based scoring rule that yields finer-grained criterion scores and reduces ties. We then couple this scoring rule with phase-specific preference rewards and an alternating RL procedure, allowing the rubric generator and judge to iteratively improve each other. The resulting reward model is deployable as a standalone pointwise evaluator and can be used directly for downstream policy optimization with standard algorithms such as  DPO~\citep{rafailov2023direct} and GRPO~\citep{shao2024deepseekmath}.

In summary, our contributions are as follows:
\begin{itemize}[leftmargin=1em]
    \item We present \name, a pointwise rubric reward modeling framework for non-verifiable domains that introduces a probability-based scoring rule and an alternating RL procedure to jointly train a rubric generator and a rubric-conditioned judge, with the RL stage relying on pairwise preference data alone.
    \item We provide theoretical analysis establishing the preference consistency of our judge-side reward, variance reduction from opposite-side averaging, and phase-wise convergence guarantees for our designed alternating training pipeline.
    \item Extensive experiments on reward modeling benchmarks and downstream policy post-training show that \name\ consistently outperforms strong baselines by 3.0\%.
\end{itemize}

\section{Related Work}
\label{sec:related}
\paragraph{Reward modeling and LLM judges.}
Reward modeling is a central component of RLHF and LLM post-training~\citep{ouyang2022training}. Early methods typically learn scalar reward models from pairwise preferences using Bradley--Terry-style objectives~\citep{liu2024skywork,wang2024interpretable}. More recent works  shifted toward generative and judge-based evaluators, including models that produce rationales or chain-of-thought judgments~\citep{chen2025rm,guo2026reward,chen2025judgelrm,mahan2024generative,yu-etal-2025-self,hong2026think}, and LLM-as-a-judge methods for open-ended evaluation~\citep{zheng2023judging,li2024llms,li2025generation,xu2025incentivizing}. Such judges can be pointwise, scoring responses independently, or pairwise, comparing responses directly. While pointwise judges are efficient and directly reusable as reward models, their scores are often too coarse to capture fine-grained preference differences in subjective tasks.

\paragraph{Rubric-based reward modeling.}
Rubric-based reward modeling improves open-ended evaluation by replacing holistic judgments with structured criteria or checklists~\citep{gunjal2025rubrics,viswanathan2026checklists,pan2026rubriceval}. Prior work has explored rubric extraction and synthesis~\citep{liu2025openrubrics,li2026rubrichub}, rubric-guided policy learning~\citep{zhou2025breaking,rezaei2025online,he2025advancedif,gunjal2025rubrics}, and rubric-conditioned judges~\citep{xu2026alternating,jia2026open}. Most related are pointwise rubric reward models~\citep{xie2025auto,zhang2025chasing}, which often depend on frontier LLMs for rubric or judge supervision and use hard aggregation over criteria. \name{} differs by jointly training the rubric generator and judge from pairwise preferences during alternating RL, while using probability-based criterion scores to obtain more discriminative pointwise rewards.

\section{Preliminaries}
\label{subsec:overview}

We consider a pointwise rubric reward model for non-verifiable tasks. Given an instruction-response pair, the evaluator first generates a rubric and then queries a rubric-conditioned judge to produce criterion-level judgments, which are aggregated into a scalar score.

\paragraph{Rubric generator.}
For an instruction \(x\), the rubric generator \(\pi_r\) with parameters \(\theta_r\) samples a rubric set \(\mathbf{r}=\{r_{k}\}_{k=1}^{K}\sim\pi_r(\cdot \mid x;\theta_r)\), where \(K\) is the number of rubric items and \(r_k\) denotes the \(k\)-th rubric item.

\paragraph{Rubric-conditioned judge.}
Given an instruction \(x\), a candidate response \(y\), and a rubric set \(\mathbf{r}\), the judge \(\pi_j\) with parameters \(\theta_j\) produces criterion-level judgments \(\mathbf{a}=\{a_{k}\}_{k=1}^{K}\) with \(a_{k}\in\{0,1\}\) and corresponding rationales \(\mathbf{e}=\{e_{k}\}_{k=1}^{K}\), drawn as \((\mathbf{a},\mathbf{e})\sim\pi_j(\cdot \mid x, y, \mathbf{r};\theta_j)\). Here, \(a_k=1\) if \(y\) satisfies \(r_k\) and \(0\) otherwise.

\paragraph{Pointwise score.}
The evaluator maps \((x,y)\) to a scalar score by aggregating the criterion-level judgments. We use the weighted average
\begin{equation}
s(x,y)
=
\frac{
\sum_{k=1}^{K} w_{k} a_{k}
}{
\sum_{k=1}^{K} w_{k}
},
\label{eq:pointwise-score}
\end{equation}
where \(w_k>0\) is a predefined prior weight assigned to rubric item \(r_k\).

\section{Method}
\label{sec:method}

We propose a framework that trains a pointwise rubric reward model in two stages: a brief SFT warm-up followed by alternating RL on pairwise preference data alone. The RL stage alternates between optimizing the rubric-conditioned judge and the rubric generator, using phase-specific rewards derived from pairwise preference data. The remainder of this section introduces a probability-based scoring rule that yields smoother pointwise scores, the phase-specific RL objectives, the alternating GRPO procedure, and finally how the learned evaluator drives downstream policy training.

\subsection{Probability-based scoring}
\label{sec:probability_based_scoring}
The pointwise score in Eq.~\eqref{eq:pointwise-score} aggregates Boolean judgments \(a_k \in \{0,1\}\), which can lead to many tied scores in practice. To obtain a fine-grained signal, we follow prior LLM-based evaluators that exploit output distributions rather than only hard decisions~\citep{liu2023g,hashemi2024llm}. Specifically, for each rubric item \(r_k\), let \(p_k^{\mathrm{true}}\) and \(p_k^{\mathrm{false}}\) denote the judge's probabilities of the ``true'' and ``false'' tokens at the verdict position. We define the criterion-level score as the probability margin
$d_k = p_k^{\mathrm{true}} - p_k^{\mathrm{false}}$.
The probability-based pointwise score is then obtained by replacing \(a_k\) with \(d_k\) in Eq.~\eqref{eq:pointwise-score}. This preserves the original aggregation structure while converting discrete decisions into continuous, more discriminative reward signals.

\subsection{Phase-specific RL objectives}

Unlike prior RL-based judge training methods that continuously rely on frontier LLMs to synthesize judge labels in the RL loop, our RL stage uses only widely available pairwise preference data. Let \(\mathcal{D}=\{(x_i,y_i^1,y_i^2,o_i^\star)\}_{i=1}^{N}\) be a preference dataset with \(o_i^\star\in\{1,2\}\) indicating the preferred response.

Pairwise preference data does not directly supervise a pointwise rubric reward model because the evaluator scores each response independently. In our alternating RL stage, we therefore use different rollout estimators for the judge and the generator phase, while both phases aim to make the pointwise evaluator recover the human preference ordering.

\paragraph{Judge-side preference-consistency reward.}
Fix the rubric generator \(\pi_r\), and let \(q_i=+1\) if \(o_i^\star=1\) and \(q_i=-1\) otherwise. For each response \(y_i^b\), \(b\in\{1,2\}\), the current judge produces \(n\) stochastic evaluation rollouts, yielding scores \(\{s_{i,m}^{b}\}_{m=1}^{n}\) and rollout averages \(\bar{s}_i^{b} = \frac{1}{n}\sum_{m=1}^{n} s_{i,m}^{b}\). We compare each rollout on one side against the opposite side's average score. For the \(m\)-th rollout of \(y_i^1\),
\begin{equation}
R_{i,m}^{1}
=
\mathbb{I}\!\left[q_i\bigl(s_{i,m}^{1}-\bar{s}_i^{2}\bigr)>0\right],
\label{eq:r1}
\end{equation}
and for the \(m\)-th rollout of \(y_i^2\),
\begin{equation}
R_{i,m}^{2}
=
\mathbb{I}\!\left[q_i\bigl(\bar{s}_i^{1}-s_{i,m}^{2}\bigr)>0\right],
\label{eq:r2}
\end{equation}
When \(o_i^\star=1\), the first reward encourages judge rollouts of \(y_i^1\) to score above \(\bar{s}_i^2\), while the second encourages judge rollouts of \(y_i^2\) to score below \(\bar{s}_i^1\); the direction reverses automatically when \(o_i^\star=2\). The corresponding judge-side pair reward is
\begin{equation}
R_i
=
\frac{1}{2n}\sum_{m=1}^{n}\left(R_{i,m}^{1}+R_{i,m}^{2}\right) \in [0,1],
\label{eq:group-reward}
\end{equation}
and the judge-side population objective is
\begin{equation}
J_j(\theta_j;\theta_r)
=
\E_{(x_i,y_i^1,y_i^2,o_i^\star)\sim \mathcal{D}}\left[R_i\right].
\label{eq:judge-objective}
\end{equation}
This phase converts pairwise preference labels into a pointwise supervision signal for the judge, even though each response is scored independently.

\paragraph{Rubric generator-side preference-ranking reward.}
Fix the judge \(\pi_j\). For each instruction \(x_i\), the rubric generator samples \(n\) rubric rollouts \(\{\mathbf r_{i,m}\}_{m=1}^{n}\). Each sampled rubric is then used by the frozen judge to evaluate both responses under the same rubric, producing \(s_{i,m}^{1}\) and \(s_{i,m}^{2}\). We assign each sampled rubric the reward
\begin{equation}
R_{i,m}^{(r)}
=
\mathbb{I}\!\left[q_i\bigl(s_{i,m}^{1}-s_{i,m}^{2}\bigr)>0\right]
-\lambda \Omega(\mathbf r_{i,m}),
\label{eq:generator-rollout-reward}
\end{equation}
where \(\Omega(\mathbf r_{i,m})\) penalizes unnecessarily long or semantically redundant rubrics; the exact length-aware bonus structure is given in Appendix~\ref{app:length-reward}. The corresponding generator-side population objective is
\begin{equation}
J_r(\theta_r;\theta_j)
=
\E_{(x_i,y_i^1,y_i^2,o_i^\star)\sim \mathcal{D}}
\left[\frac{1}{n}\sum_{m=1}^{n} R_{i,m}^{(r)}\right].
\label{eq:generator-objective}
\end{equation}
The generator is rewarded for producing concise, non-redundant rubrics that enable the frozen judge to rank the preferred response higher. We optimize the judge and generator in alternating update phases until convergence.

\subsection{Alternating optimization}

\paragraph{SFT warm-up.}
Before alternating RL, we initialize both modules with a brief supervised stage. We adopt the high-quality rubrics from \textsc{OpenRubrics}~\citep{liu2025openrubrics} as-is and generate judge labels for each (instruction, response) pair with GPT-5-mini. We then filter the warm-up data to keep only pairs whose chosen response receives a higher aggregated pointwise score than the rejected response, and use the resulting data to SFT both the rubric generator \(\pi_r\) and the rubric-conditioned judge \(\pi_j\). After this one-time warm-up, the alternating RL stage uses only pairwise preference data and no further frontier-LLM annotation.

\paragraph{Judge phase.}
We fix \(\pi_r\) and optimize \(\pi_j\) with GRPO on \(J_j(\theta_j;\theta_r)\). For each preference pair, the GRPO group contains the \(2n\) rollout rewards \(\{R_{i,m}^{1},R_{i,m}^{2}\}_{m=1}^{n}\). Since \(\pi_r\) is frozen, the rubrics for each training instance can be pre-generated once and reused throughout the judge update, which substantially reduces wall-clock cost.

\paragraph{Rubric-generator phase.}
We then freeze the judge \(\pi_j\) and optimize the rubric generator \(\pi_r\) with GRPO on \(J_r(\theta_r;\theta_j)\). For each preference pair, the GRPO group consists of the \(n\) sampled rubrics and their rewards \(\{R_{i,m}^{(r)}\}_{m=1}^{n}\). The reward is assigned to each rubric rollout, as a rubric defines the evaluation criteria for both responses and is scored by whether the frozen judge ranks the preferred response higher under those criteria.

\subsection{Policy Training}
After training, the pointwise rubric reward model serves as a standalone evaluator for policy optimization. In the offline setting, we use it to construct DPO training pairs~\citep{rafailov2023direct} by scoring candidate responses for each instruction and selecting the highest- and lowest-scoring responses as the chosen and rejected samples, respectively. In the online setting, we use the pointwise score as the reward signal for GRPO~\citep{shao2024deepseekmath}.

\section{Theoretical Analysis}
\label{sec:analysis}
We summarize four theoretical properties of our method and defer all proofs to Appendix~\ref{app:proofs}. The first two results analyze the judge-side reward \(J_j\). The last two are phase-wise GRPO guarantees: they apply to either a judge phase or a generator phase with the other module held fixed.

\subsection{Preference consistency of the judge-side reward}
Fix an instruction \(x\) and two candidate responses \(y^+\) and \(y^-\) such that the human prefers \(y^+\) over \(y^-\). Let the rollout-level scalar scores satisfy \(S^+=\mu^+ + \xi^+\) and \(S^- = \mu^- + \xi^-\), and define \(\Delta:=\mu^+-\mu^-\). Let \(J_n\) denote the judge-side preference recovery probability induced by our rollout comparison rule.

\begin{theorem}[Preference consistency]\label{thm:pref-consistency}
Under independent, continuous, zero-mean, symmetric rollout noise, \(J_n>\frac{1}{2}\) if and only if \(\Delta>0\), and \(J_n=\frac{1}{2}\) when \(\Delta=0\).
\end{theorem}
Theorem~\ref{thm:pref-consistency} states that the judge-side reward is calibrated for recovering the human preference.

\subsection{Variance reduction from opposite-side averaging}
\begin{theorem}[Variance reduction from opposite-side averaging]\label{thm:variance-reduction}
Under Gaussian rollout noise and \(\Delta>0\), comparing each rollout against the opposite response's rollout average yields a strictly higher preference recovery probability than a single pairwise comparison whenever \(n>1\). In the homoscedastic case \(\sigma_+=\sigma_-=\sigma\),
\begin{equation}
\setlength{\abovedisplayskip}{3pt}
\setlength{\belowdisplayskip}{3pt}
J_n
=
\Phi\!\left(\frac{\Delta}{\sigma\sqrt{1+1/n}}\right),
\label{eq:body-variance-short}
\end{equation}
where \(\Phi\) is the standard normal CDF. Thus, \(J_n\) increases monotonically with \(n\). The exact heteroscedastic expression for \(\sigma_+\ne\sigma_-\) is given in Appendix~\ref{app:proofs}.
\end{theorem}

Theorem~\ref{thm:variance-reduction} shows that opposite-side averaging improves judge reliability by reducing comparison variance.

\subsection{Optimization guarantees}
For the next two results, let \(b\in\{j,r\}\) denote a phase of the alternating algorithm: \(b=j\) corresponds to judge updates with \(\theta_r\) frozen, and \(b=r\) corresponds to generator updates with \(\theta_j\) frozen. Let \(J_b\) be the corresponding population objective (\(J_j\) or \(J_r\)), and let \(L_{G,b}\) be the one-step on-policy GRPO surrogate for that phase. The effective rollout group sizes are \(G_j=2n\) for the judge phase and \(G_r=n\) for the generator phase. Here \(T\) denotes the number of GRPO updates performed in phase \(b\), \(\theta_b^t\) is the active-module parameter at phase step \(t\), and \(\theta_{-b}\) denotes the frozen parameter of the other module.

\begin{theorem}[Phase-wise surrogate stationary-point convergence of GRPO]\label{thm:surrogate-grpo}
Fix a phase \(b\in\{j,r\}\) and hold the other module constant. Under the smoothness, unbiased-gradient, bounded-variance, and diminishing-step-size assumptions in Appendix~\ref{app:proofs},
\begin{equation}
\setlength{\abovedisplayskip}{3pt}
\setlength{\belowdisplayskip}{3pt}
\frac{1}{T}\sum_{t=0}^{T-1}\E\norm{\nabla_{\theta_b}L_{G,b}(\theta_b^t;\theta_{-b})}^2
=
O(T^{-1/2}).
\label{eq:body-surrogate-short}
\end{equation}
Hence GRPO attains first-order stationarity for the phase-specific surrogate.
\end{theorem}

\begin{theorem}[Phase-wise stationary neighborhood for the true objective]\label{thm:true-objective-neighborhood}
Fix a phase \(b\in\{j,r\}\) and hold the other module constant. If, in addition, the surrogate-to-true gradient gap for that phase is bounded as specified in Appendix~\ref{app:proofs}, then
\notag\begin{align}
\setlength{\abovedisplayskip}{3pt}
\setlength{\belowdisplayskip}{3pt}
\frac{1}{T}\sum_{t=0}^{T-1}&\E\norm{\nabla_{\theta_b}J_b(\theta_b^t;\theta_{-b})}^2
=
O(T^{-1/2}) \\&+ O(G_b^{-1}) + O(B_{\mathrm{stale}}^{(b)}) + O(B_{\mathrm{clip}}^{(b)}).
\label{eq:body-true-short}
\end{align}
Here \(G_b\) is the effective rollout group size in phase \(b\), with \(G_j=2n\) for the judge phase and \(G_r=n\) for the generator phase. The \(O(G_b^{-1})\) term absorbs the finite-group error \(B_{\mathrm{grp}}^{(b)}/G_b\), while \(B_{\mathrm{stale}}^{(b)}\) and \(B_{\mathrm{clip}}^{(b)}\) denote stale-policy and clipping bias terms. Hence the rollout-dependent term is \(O(n^{-1})\) in both phases.
\end{theorem}

\section{Experiment}

\subsection{Datasets and Experiment Settings}

\noindent \textbf{Training data.} We train \name\ on the general-domain pairwise preference data from \textsc{OpenRubrics}~\citep{liu2025openrubrics}. We divide the dataset into equally sized, disjoint subsets and assign one subset to each outer iteration of the alternating procedure.

\noindent \textbf{Backbone and variants.}
Both the rubric generator and the judge are initialized from Qwen-3-8B~\citep{yang2025qwen3}. At inference time, \name\ follows the two-stage pointwise rubric-judging pipeline described in Section~\ref{sec:method}. We additionally report two variants: a voting@5 ensemble, where five independent rubric-judging runs are aggregated by majority voting; and \name\ w/o RL, a pointwise rubric reward model that uses only the SFT warm-up checkpoint without the alternating RL stage, included to isolate the contribution of RL.

\noindent \textbf{Baselines.}
For reward-model evaluation, we compare \name\ with generative judges (RM-R1~\citep{chen2025rm}, RRM~\citep{guo2026reward}, JudgeLRM~\citep{chen2025judgelrm}), prompted Qwen-3-8B rubric+judge pipelines, and trained rubric-based baselines (\textsc{Rubric-RM}~\citep{liu2025openrubrics}, {\name} w/o RL, \textsc{RIFL}~\citep{he2025advancedif}), with black-box APIs included for reference. For downstream policy training on Qwen2.5-7B-Instruct~\citep{qwen2.5}, we compare against scalar reward models (Skywork~\citep{liu2024skywork}, ArmoRM~\citep{wang2024interpretable}), AI feedback data (UltraFeedback~\citep{cui2023ultrafeedback}), and rubric/checklist baselines (RLCF~\citep{viswanathan2026checklists}, \textsc{Rubric-RM}, {\name} w/o RL, \textsc{RIFL}). Full baseline descriptions are given in Appendix~\ref{app:baselines}.

\noindent \textbf{Evaluation benchmarks and metrics.}
We evaluate \name\ in two settings. For \emph{reward-model evaluation}, we use RewardBench (Chat/Chat-Hard)~\citep{lambert2025rewardbench}, FollowBench~\citep{jiang2024followbench}, PPE-IFEval~\citep{frick2025evaluate}, InfoBench~\citep{qin2024infobench}, IFBench~\citep{peng2025agentic}, RM-Bench~\citep{liu2025rm}, RewardBench2 (Precise-IF/Focus)~\citep{malik2025rewardbench}, and HelpSteer3~\citep{wang2026helpsteer3}. For \emph{downstream policy evaluation}, we use IFEval~\citep{frick2025evaluate}, InfoBench~\citep{qin2024infobench}, IFBench~\citep{peng2025agentic}, Arena-Hard~\citep{li2024crowdsourced}, AlpacaEval2~\citep{dubois2024length}, Creative Writing Benchmark v3~\citep{paech2025eq}, and WildBench~\citep{lin2025wildbench}. Evaluation construction and metric details are deferred to Appendix~\ref{app:eval-details}.

\begin{figure}[t!]
  \centering
   \vspace{-0.5ex}
  \includegraphics[width=0.6\linewidth]{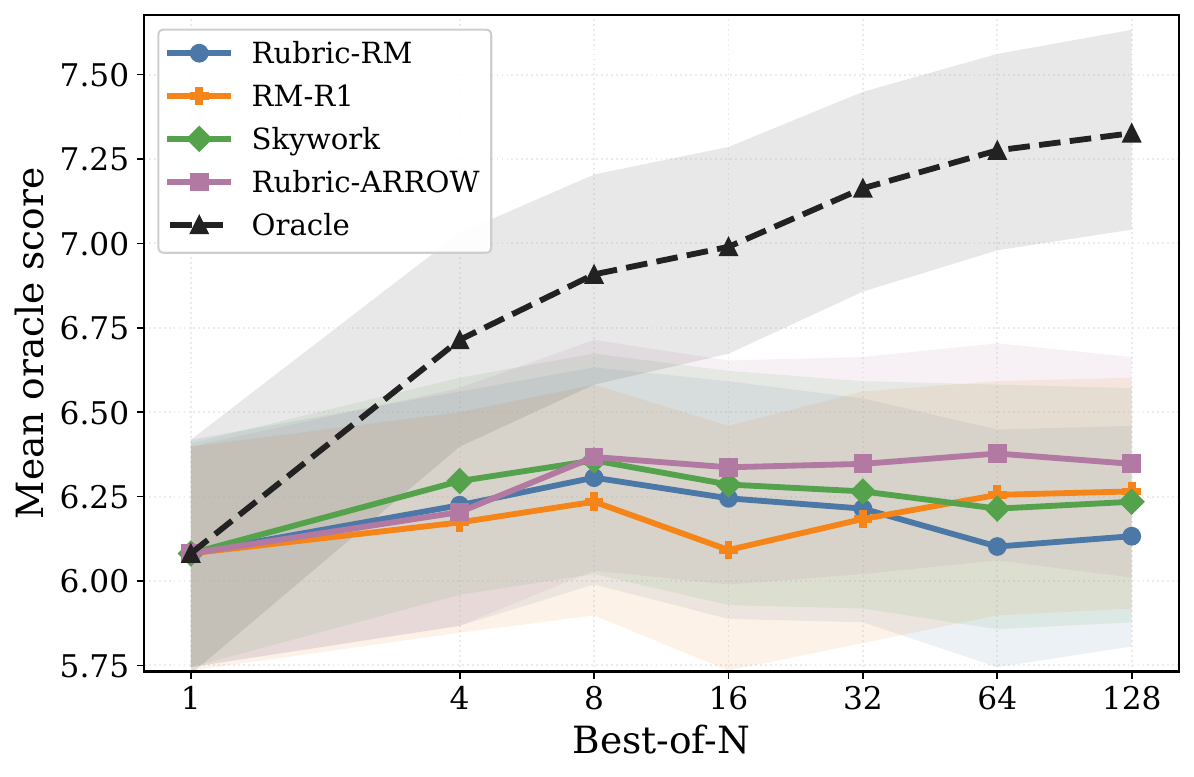}
  \caption{Best-of-\(N\) selection quality on WildBench: each method selects one response from the first \(N\) sampled candidates, and we report the mean oracle score. The Oracle curve marks the per-prompt best candidate (upper bound); shaded regions show 95\% bootstrap confidence intervals over prompts. Failure to track the Oracle curve at larger \(N\) indicates reward-model overoptimization. \vspace{-1ex}}
  \label{fig:bestofn}
  \vspace{-1ex}
\end{figure}

\begin{table*}[t!]
\centering
\renewcommand\arraystretch{0.9}
\caption{Comparison of judge and reward models across benchmarks. RewardBench2 reports its Precise IF and Focus dimensions; the API pipelines use GPT-4.1-Mini as the rubric generator and Gemini-2.5-Flash-Lite as the judge. Best results in \textbf{bold}.
\vspace{-1.5ex}
}
\label{tab:main_results}
\resizebox{0.99\textwidth}{!}{%
\begin{tabular}{l cc cccc c cc cc}
\toprule
\multirow{2.5}{*}{}
& \multicolumn{2}{c}{\bf RewardBench}
& \multicolumn{4}{c}{\bf IF Evaluation Benchmarks}
& \bf RM-Bench
& \multicolumn{2}{c}{\bf RewardBench2}
& \multirow{2.5}{*}{\bf HelpSteer3}
& \multirow{2.5}{*}{\bf Avg.} \\
\cmidrule(lr){2-3}  \cmidrule(lr){4-7} \cmidrule(lr){8-8} \cmidrule(lr){9-10}
                            & Chat & Chat Hard  & FollowBench   &   PPE-IFEval  &   InfoBench   &   IFBench     &    Chat        & Precise IF    &   Focus  &      &   \\
\midrule
\multicolumn{11}{l}{\it Black-box LLMs (For reference only)} \\
\midrule
Claude-3.5-Sonnet           & 96.4 & 74.0 & -- & 58.0 & -- & -- & 62.5 & 38.8 & 87.0 & -- & - \\
Gemini-2.5-Flash            & 95.0 & 83.3 & 86.0 & 75.0 & 85.6 & 69.3 & 78.5 & 57.5 & 84.1 & 70.6 & 78.5  \\
API Prompting (Rubric+Judge pairwise) & 79.6 & 79.2 & 83.2 & 61.0 & 82.2 & 66.2 & 67.9 & 42.5 & 79.6 & 71.4 & 71.3 \\
API Prompting (Rubric+Judge pointwise)& 51.7 & 61.5 & 60.5 & 26.4 & 44.8 & 33.6 & 37.8 & 14.8 & 50.4 & 39.2 & 42.1 \\
API Prompting (direct Judge)          & 89.6 & 71.2 & 81.7 & 59.2 & 72.9 & 60.4 & 67.2 & 13.2 & 63.4 & 70.3 & 64.9 \\
\midrule
\multicolumn{11}{l}{\it Larger White-box LLMs (For reference only)} \\
\midrule
RM-R1-14B (Qwen-2.5-Inst)   & 73.5 & 79.8 & 84.0 & 59.0 & 85.5 & 60.8 & 73.2 & 23.8 & 84.6	& 74.8 & 69.9 \\
RM-R1-14B (DeepSeek-Dist)   & 90.3 & 78.9 & 89.9 & 61.2 & 82.4 & 59.0 & 71.4 & 30.6 & 79.0 & 74.6 & 71.7 \\
RM-R1-32B (Qwen-2.5-Inst)   & 95.3 & 80.3 & 84.9 & 60.4 & 86.1 & 60.4 & 75.3 & 33.1 & 84.2 & 72.9 & 73.3\\
RM-R1-32B (DeepSeek-Dist)   & 95.3 & 83.1 & 89.2 & 63.2 & 85.0 & 58.6 & 74.2 & 36.9 & 79.2 & 75.6 & 74.0\\
RRM-32B                     & 94.7 & 81.1 & 85.7 & 60.2 & 84.4 & 60.8 & 73.9 & 34.4 & 83.6 & 75.4 & 73.4\\
\midrule
\multicolumn{11}{l}{\it White-box Judge/Reward LLMs} \\
\midrule
RM-R1-7B (Qwen-2.5-Inst)    & 83.0 & 70.0 & 56.3 & 55.2	& 71.3 & 55.2 & 64.2 & 20.6 & 76.2 & 65.2 & 61.7 \\
RM-R1-7B  (DeepSeek-Dist)   & 85.3 & 67.3 & 69.7 & 51.0 & 70.3 & 56.5 & 62.2 & 13.8 & 55.4 & 62.6 & 59.4 \\
RRM-7B	                    & 77.7 & 69.5 & 65.5 & 51.0 & 68.2 & 53.2 & 59.9 & 10.0 & 60.4	& 62.4 & 57.8 \\
JudgeLRM-7B                 & \bf 92.1 & 56.1 & 79.8 & 46.0 & 62.7 & 47.5 & 55.4 & 9.4  & 29.1 & 60.2 & 53.8 \\
\midrule
\multicolumn{11}{l}{\it Rubric-based Methods} \\
\midrule
Qwen-3-8B (Rubric+Judge pairwise)  & 73.9 & 63.6 & 63.0 & 53.8 & 74.6 & 55.6 & 64.2 & 21.9 & 56.6 & 61.8 & 58.9 \\
Qwen-3-8B (Rubric+Judge pointwise) & 69.3 & 66.4 & 57.5 & 52.6 & 69.9 & 50.2 & 61.2 & 27.0 & 62.9 & 55.5 & 57.3 \\
{\textsc{Rubric-RM}}               & 88.2 & 74.1 & 76.1 & 67.0 & 80.8 & 65.4 & 65.7 & 34.4 & 82.2 & 67.0 & 70.1 \\
{\textsc{Rubric-RM}}-voting@5      & 89.9 & 75.4 & 81.5 & 70.8 & 83.8 & 67.1 & 67.0 & 40.0 & \bf 86.5 & 67.5 & 73.0\\
{\name} w/o RL          & 87.4 & 71.0 & 77.8 & 69.5 & 75.9 & 65.5 & 64.8 & 40.3 & 73.3 & 67.5 & 69.3 \\
{\name} w/o RL-voting@5 & 87.7 & 73.3 & \bf 84.0 & 72.2 & 79.7 & 68.5 & 66.9 & 43.1 & 77.0 & 67.5 & 72.0 \\
\textsc{RIFL}           & 89.6 & 73.3 & 81.2 & 73.3 & 78.4 & 69.2 & 64.8 & 42.3 & 76.7 & 67.7 & 71.7 \\
\textsc{RIFL}-voting@5  & 89.7 & 74.1 & \bf 84.0 & 75.8 & 80.1 & 69.8 & 64.3 & 41.9 & 77.2 & 67.7 & 72.5 \\
\rowcolor{purple!10}
{\name}                 & 89.1 & 79.2 & 78.2 & 72.6 & 82.1 & 68.2 & 67.8 & 42.9 & 80.8 & 70.2 & 73.1 \\
\rowcolor{purple!10}
{\name}-voting@5        & 90.8 & \bf 81.8 & 79.8 & \bf 76.0 & \bf 84.0 & \bf 73.2 & \bf 68.6 & \bf 45.0 & 84.9 & \bf 72.0 & \bf 75.6 \\
\bottomrule
\end{tabular}
}
\vspace{-1ex}
\end{table*}
\subsection{Performance of {\name}}

Table~\ref{tab:main_results} shows that {\name} provides a strong white-box pointwise reward model for open-ended evaluation. Among white-box methods, {\name} achieves the best average performance, improving over the strongest rubric baseline \textsc{Rubric-RM} from 70.1 to 73.1, and further to 75.6 with voting@5. The gains are most pronounced on instruction-following and chat-oriented benchmarks, where {\name}-voting@5 achieves the best results on Chat Hard, PPE-IFEval, InfoBench, and IFBench. The improvement over the no-RL variant further confirms the benefit of our alternating RL procedure, indicating that jointly optimizing the rubric generator and judge yields more discriminative reward signals. Notably, {\name}-voting@5 also outperforms prompting-based judges, suggesting that learned rubric-conditioned evaluation is more reliable than direct prompting.

We next study robustness under best-of-\(N\) selection on 100 WildBench prompts from instruction-following and chat categories, using frontier-LLM scores under the dataset rubric as the oracle reference. As shown in Fig.~\ref{fig:bestofn}, Rubric-RM and Skywork begin to degrade after \(N=8\), suggesting sensitivity to reward overoptimization. In contrast, {\name} continues to closely track the oracle upper bound, indicating that its rewards remain well aligned when used for  offline selection.

\subsection{Ablation Study}

We perform two ablation studies to better understand where the improvements of {\name} come from. First, we remove the length penalty to examine whether controlling rubric verbosity is important for reliable rubric-conditioned judging. Second, we evaluate probability-based scoring against a random-probability control to verify whether the model-derived confidence scores provide meaningful tie-breaking signals.
\paragraph{Length penalty.}
Removing the length penalty consistently hurts performance, reducing the average score from 73.1 to 72.3 and voting@5 from 75.6 to 74.3 (Table~\ref{tab:ablation}). The drop is most visible on stricter instruction-following benchmarks such as Chat Hard and PPE-IFEval, suggesting that length regularization helps prevent overly verbose or redundant rubrics and improves reward accuracy.

\paragraph{Probability-based scoring.}
Table~\ref{tab:prob-ablation} shows that probability-based scoring breaks ties without sacrificing correctness. On InfoBench and IFBench, it resolves 98/140 and 68/115 tied cases while preserving most originally correct predictions (305/312 and 234/255). In contrast, a random-probability control resolves a similar number of ties but corrupts many correct decisions, preserving only 160/312 and 130/255. Thus, the improvement comes from meaningful judge confidence rather than arbitrary tie-breaking.

\begin{table*}[!t]
\centering
\renewcommand\arraystretch{0.95}
\caption{Ablation on the length penalty. Best results in \textbf{bold}.
\vspace{-1ex}
}
\label{tab:ablation}
\resizebox{0.99\textwidth}{!}{%
\begin{tabular}{l cc cccc c cc cc}
\toprule
\multirow{2.5}{*}{}
& \multicolumn{2}{c}{\bf RewardBench}
& \multicolumn{4}{c}{\bf IF Evaluation Benchmarks}
& \bf RM-Bench
& \multicolumn{2}{c}{\bf RewardBench2}
& \multirow{2.5}{*}{\bf HelpSteer3}
& \multirow{2.5}{*}{\bf Avg.} \\
\cmidrule(lr){2-3}  \cmidrule(lr){4-7} \cmidrule(lr){8-8} \cmidrule(lr){9-10}
                            & Chat & Chat Hard  & FollowBench   &   PPE-IFEval  &   InfoBench   &   IFBench     &    Chat        & Precise IF    &   Focus  &      &   \\
\midrule
{{\name} w/o length penalty}          & 87.5 & 75.6 & 79.7 & 69.6 & 81.5 & 69.0 & 66.1 & 43.9 & 80.6 & 69.3 & 72.3 \\
{{\name} w/o length penalty}-voting@5 & \bf 90.8 & 78.5 & \bf 81.5 & 73.2 & \bf 86.7 & 71.6 & \bf 68.7 & 38.1 & 82.6 & 71.2 & 74.3 \\
\rowcolor{purple!10}
{\name}                               & 89.1 & 79.2 & 78.2 & 72.6 & 82.1 & 68.2 & 67.8 & 42.9 & 80.8 & 70.2 & 73.1 \\
\rowcolor{purple!10}
{\name}-voting@5                      & \bf 90.8 & \bf 81.8 & 79.8 & \bf 76.0 & 84.0 & \bf 73.2 & 68.6 & \bf 45.0 & \bf 84.9 & \bf 72.0 & \bf 75.6 \\
\bottomrule
\end{tabular}
}
\end{table*}

\begin{table}[t]
\centering
\renewcommand\arraystretch{0.95}
\caption{
Ablation study on probability-based scoring on InFoBench~\citep{qin2024infobench} and IFBench~\citep{peng2025agentic}
\vspace{-1ex}}
\resizebox{0.7\columnwidth}{!}{%
\begin{tabular}{llccc}
\toprule
\bf Dataset & \bf Method
& \bf Tie $\rightarrow$ Correct
& \bf Correct Preserved
& \bf Wrong $\rightarrow$ Correct \\
\midrule
\multirow{2}{*}{InfoBench}
& Probability
& 98 / 140
& 305 / 312
& 5 / 36 \\
& Random
& 63 / 140
& 160 / 312
& 12 / 36 \\
\midrule
\multirow{2}{*}{IFBench}
& Probability
& 68 / 115
& 234 / 255
& 13 / 74 \\
& Random
& 67 / 115
& 130 / 255
& 38 / 74 \\
\bottomrule
\end{tabular}
}
\label{tab:prob-ablation}
\end{table}

\begin{table}[!t]
\centering
\caption{Trained-policy comparison on IFEval and InfoBench. Asterisks (\(^\star\)) denote results taken from \citet{viswanathan2026checklists,liu2025openrubrics}; underlined results were reproduced by us. Best in \textbf{bold}. \vspace{-1ex}
}
\label{tab:ifeval}
\resizebox{0.8\columnwidth}{!}{%
\arrayrulecolor{black!50}
\begin{tabular}{l|cc|cc|c|c}
\toprule
\multirow{2}{*}{\textbf{Model}}
& \multicolumn{2}{c|}{\textbf{IFEval (Prompt)}}
& \multicolumn{2}{c|}{\textbf{IFEval (Inst.)}}
& {\textbf{IFEval}}
& {\textbf{InfoBench}} \\
\cmidrule(lr){2-3}  \cmidrule(lr){4-5} \cmidrule(lr){6-6} \cmidrule(lr){7-7}
& \textbf{Loose} & \textbf{Strict} & \textbf{Loose} & \textbf{Strict} & \textbf{AVG} &  \textbf{AVG} \\
\midrule
GPT-4 (0314)${\star}$                              & 79.3 & 76.9 & 85.4 & 83.6 & 81.3 & 87.3 \\
AutoIF~\citep{dong2025self}               & 56.9 & 47.1 & 67.0 & 57.6 & 57.2 & 80.6 \\
UltraIF~\citep{an2025ultraif}             & 75.4 & 71.3 & 83.0 & 79.4 & 77.3 & 80.7 \\
RAIF~\citep{qin2026incentivizing}  &-- & -- & -- & -- & 70.1 & 82.7 \\
\midrule
Qwen2.5-7B-Instruct${\star}$          & 75.0 & 72.5 & 81.8 & 79.9 & 77.3 & 78.1 (\underline{76.0}) \\
+ SFT (Distilled)${\star}$            & 66.8 & 64.1 & 75.3 & 72.8 & 69.8 & 72.5 \\
+ DPO (via Skywork)${\star}$          & 75.7 & 68.0 & 83.2 & 78.5 & 76.0 & 82.0 \\
+ DPO (via ArmoRM)${\star}$           & 73.8 & 70.2 & 81.7 & 78.3 & 76.0 & 83.5 \\
+ DPO (via Ultrafbk.)${\star}$        & 71.5 & 69.1 & 79.9 & 77.7 & 74.6 & 80.0 \\
+ DPO (via AI Judge)${\star}$         & 73.0 & 68.9 & 80.9 & 77.8 & 75.2 & 76.1 \\
+ DPO (via RLCF)${\star}$             & 77.3 & 72.6 & 84.1 & 80.3 & 78.6 & 84.1 (\underline{81.5}) \\
+ IterDPO (via RLCF)            & 78.2& 74.3& 84.5& 81.1&79.5 &81.8   \\
+ DPO (via \textsc{Rubric-RM})${\star}$            & 78.2 & 73.9 & 84.5 & 81.2 & 79.5 & 83.0 \\
+ IterDPO (via \textsc{Rubric-RM})            &77.6&74.1&84.3&81.7&79.4&83.3  \\
+DPO (via {\name w/o RL}) & 76.0& 72.8& 82.9& 80.3&78.0 &82.3   \\
+IterDPO (via {\name w/o RL}) & 76.7& 74.1& 83.2& 81.2&79.4 &82.4   \\
+DPO (via \textsc{RIFL}) & 77.5& 74.1& 84.3& 81.8&79.4 &82.8   \\
+IterDPO (via \textsc{RIFL}) & 78.0& 74.9& 84.2& 81.5&79.7 &\bf84.1   \\
\midrule
\rowcolor{purple!10}
+ DPO (via \name)            & 79.7 & \bf75.6  & 85.3 & \bf82.1 & \bf80.7 & 83.6 \\
\rowcolor{purple!10}
+ IterDPO (via \name)            & \bf79.9 & 74.5 & \bf85.4 & 81.4 & 80.3  & 83.7 \\
\bottomrule
\end{tabular}%
}
\vspace{-1ex}
\end{table}

\begin{table}[t]
\centering
\caption{Trained-policy comparison on Arena-Hard and AlpacaEval under vanilla and style/length-controlled settings. Asterisks (\(^\star\)) denote results from \citet{viswanathan2026checklists,rezaei2025online,liu2025openrubrics}. Best in \textbf{bold}. \vspace{-0.5ex}
}
\label{tab:alpaca}
\resizebox{0.8\textwidth}{!}{%
\arrayrulecolor{black!50}
\begin{tabular}{l|cc|cc|c}
\toprule
\multirow{2}{*}{\textbf{Model}}
& \multicolumn{2}{c|}{\textbf{Arena-Hard}}
& \multicolumn{2}{c|}{\textbf{AlpacaEval}}
& \multirow{2}{*}{\textbf{AVG}} \\
\cmidrule(lr){2-3} \cmidrule(lr){4-5}
& \textbf{Vanilla} & \textbf{Style-Con}
& \textbf{Vanilla} & \textbf{Length-Con} &  \\
\midrule
GPT-4 (0314)${\star}$                 & 50.0 & 50.0 & 22.1 & 35.3 & 39.4 \\
UltraIF~\citep{an2025ultraif}   & 31.4 & -- & -- & -- & -- \\
\midrule
Qwen2.5-7B-Instruct${\star}$         & 51.3 & 42.8 & 33.5 & 36.2 & 41.0 \\
+ SFT (Distilled)${\star}$           & 32.6 & 29.2 & 36.1 & 33.3 & 32.8 \\
+ DPO (via Skywork)${\star}$         & 55.1 & 50.3 & 44.8 & 41.5 & 47.9 \\
+ DPO (via ArmoRM)${\star}$           & 50.8 & 46.4 & 37.6 & 38.1 & 43.2 \\
+ DPO (via Ultrafbk.)${\star}$        & 52.8 & 47.9 & 33.7 & 38.7 & 43.3 \\
+ DPO (via AI Judge)${\star}$         & 51.0 & 44.4 & 28.8 & 33.4 & 39.4 \\
+ DPO (via RLCF)${\star}$             & 54.6 & 48.4 & 36.2 & 37.1 & 44.1 \\
+ IterDPO (via RLCF)            & 51.1&54.6&38.9&39.2&46.0 \\
+ DPO (via \textsc{Rubric-RM})${\star}$            & 52.9 & 53.1 & 47.0 & 41.3 & 48.6 \\
+ IterDPO (via \textsc{Rubric-RM})            &56.3&56.7&50.1&42.0&51.3  \\
+ RL (via \textsc{OnlineRubrics})${\star}$            & 56.5 & -- &\bf  55.0 & 30.4& -- \\
+ DPO (via {\name w/o RL}) &52.6&54.3&40.7&42.1&47.4  \\
+ IterDPO (via {\name w/o RL}) &55.8&56.9&42.7&42.1&49.4  \\
+ DPO (via \textsc{RIFL}) &52.4&53.4&43.2&41.1&47.5  \\
+ IterDPO (via \textsc{RIFL}) &55.2&55.7&45.6&42.4&49.7  \\
\midrule
\rowcolor{purple!10}
+ DPO (via \name)            & 53.5 & 54.3 & 49.1 & 45.7 & 50.7 \\
\rowcolor{purple!10}
+ IterDPO (via \name)            & \bf 57.5 &  \bf57.7 & 51.4 & \bf 45.5 & \bf 53.0 \\
\bottomrule
\end{tabular}%
}
\vspace{-1ex}
\end{table}

\begin{table}[t]
\centering
\caption{Trained-policy comparison on WildBench, with task-specific and macro WB scores. Asterisks (\(^\star\)) denote results from \citet{wang2025drift,xu2026alternating}. Best in \textbf{bold}.
}
\label{tab:wildbench}
\resizebox{0.8\columnwidth}{!}{%
\begin{tabular}{@{}l|cccccc@{}}
\toprule
\textbf{Method} & \textbf{Creative} & \textbf{Planning} & \textbf{Math} & \textbf{Info seeking} & \textbf{Coding} & \textbf{WB Score} \\ \midrule
Claude-3.5-Sonnet (20240620)$^{\star}$                       & 55.6     & 55.6     & 50.2 & 55.5         & 56.5   & 54.7     \\
GPT-4-turbo (20240409)$^{\star}$                             & 58.7     & 56.2     & 51.0 & 57.2         & 55.1   & 55.2     \\
GPT-4o-mini (20240718)$^{\star}$                           & 60.1     & 58.2     & 54.0 & 57.4         & 57.2   & 57.1     \\ \midrule
Qwen2.5-7B-Instruct$^{\star}$                     & 50.1     & 51.8     & 47.1 & 50.7         & 45.0   & 48.7     \\
+ DRIFT$^{\star}$                                  & 52.5     & 53.2     & 50.6 & 52.4         & 50.3   & 51.7    \\
+ SPIN$^{\star}$ & 43.3 & 45.5 & 41.6 & 46.3 & 39.1 & 42.9 \\
+ IterDPO$^{\star}$ (via OpenAssistant) & 46.8 & 48.6 & 44.5 & 48.0 & 44.3 & 46.3 \\
+ DPO$^{\star}$ (via RLCF)                         & 51.4     & 52.7     & 49.0 & 51.3         & 48.8   & 50.5     \\
+ IterDPO$^{\star}$ (via RLCF)            & 51.9&52.6&47.8&51.4&46.5&49.7 \\
+ DPO (via \textsc{Rubric-RM})                    & 54.8        & 55.5        & 51.5    & 54.1            & 52.9      & 53.6       \\
+ IterDPO (via \textsc{Rubric-RM})            &57.0&56.2&50.6&54.9&52.8&54.0  \\
+ DPO (via {\name w/o RL})&53.3	&53.9	&46.8	&54.4	&51.6	&51.8  \\
+ IterDPO (via {\name w/o RL})&54.5&	55.7&	51.2&	54.6&	52.8&	53.2    \\
+ DPO (via \textsc{RIFL})&55.0&	55.2&	50.7&	54.0&	51.0&	52.6   \\
+ IterDPO (via \textsc{RIFL})&55.1&	57.0&	52.8&	55.1&	52.5&	54.0   \\
\midrule
\rowcolor{purple!10}
+ DPO (via {\name})     & 54.9&	55.8&	52.2&	54.8&	51.4&	53.6        \\
\rowcolor{purple!10}
+ IterDPO (via {\name}) & \bf56.4&	\bf57.9&	\bf52.6&	\bf56.6&	\bf53.3&	\bf55.2     \\
\bottomrule
\end{tabular}%
}
\vspace{-1ex}
\end{table}
\subsection{Performance of offline RL}

We investigate whether the gains of {\name} generalize to downstream offline policy learning.

\paragraph{Instruction-Following Evaluation.}
Table~\ref{tab:ifeval} and Fig.~\ref{fig:ifbench} show that policies trained with \name-derived rewards yield consistent instruction-following gains. On IFEval, DPO via {\name} lifts the average score to 80.7, outperforming both Rubric-RM and the {\name} variant without RL, while remaining competitive on InfoBench (83.6 with DPO, 83.7 with IterDPO). The advantage is more pronounced on IFBench (Fig.~\ref{fig:ifbench}): {\name} reaches 37.4 with DPO and 36.7 with IterDPO, outperforming Rubric-RM under both settings.

\paragraph{Human preference alignment evaluation.}
\name-trained rewards also transfer to preference-alignment benchmarks (Table~\ref{tab:alpaca}, Table~\ref{tab:wildbench}). On Arena-Hard and AlpacaEval, DPO via {\name} reaches an average of 50.7 and IterDPO further improves it to 53.0, achieving the best results on Arena-Hard (both vanilla and style-controlled) and on AlpacaEval length-controlled. On WildBench, IterDPO via {\name} reaches the best macro score of 55.2 (53.6 with DPO), surpassing IterDPO via Rubric-RM by 2.2\% relative.

\paragraph{Creative writing evaluation.}
On the Creative Writing Benchmark v3 (Fig.~\ref{fig:cw}), policies trained with {\name} lead all compared methods: DPO via {\name} reaches 39.8 and IterDPO further improves the score to 40.5. \name-based optimization also outperforms creative-writing baselines such as RaR, Rubric-based RL, and RuscaRL, as well as Rubric-RM-based DPO and IterDPO. This indicates that \name-derived rewards generalize beyond instruction-following and preference benchmarks to open-ended generation.

\begin{figure}[t!]
  \centering
   \vspace{-0.5ex}
  \includegraphics[width=0.95\linewidth]{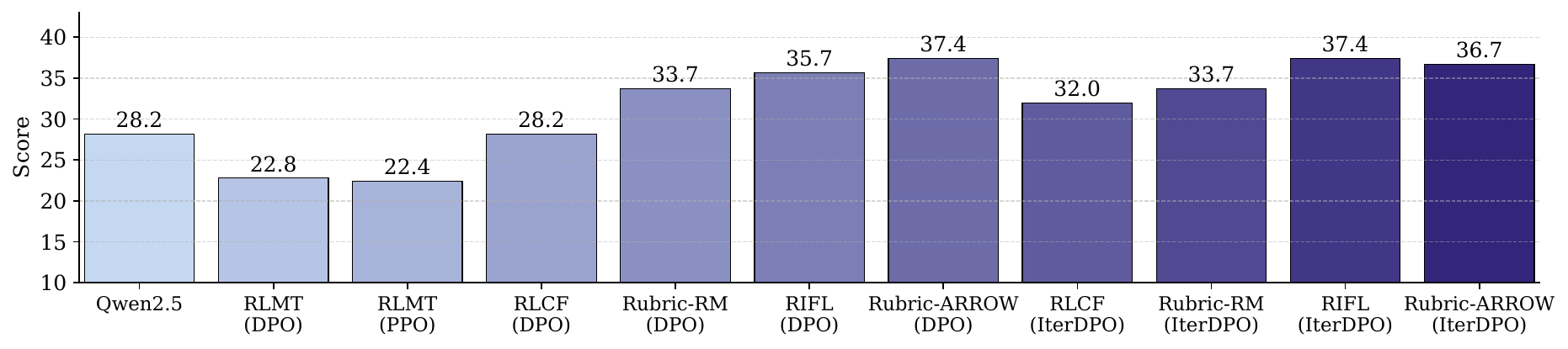}
\caption{Trained-policy comparison on IFBench. Baseline results (except Rubric-RM IterDPO) are from OpenRubrics~\cite{liu2025openrubrics}; two baselines are omitted for readability (see Fig.~\ref{fig:ifbench_full} in Appendix~\ref{app:experiment} for the full version).\vspace{-1ex}}
  \label{fig:ifbench}
\end{figure}

\begin{figure}[t!]
  \centering
  \includegraphics[width=0.95\linewidth]{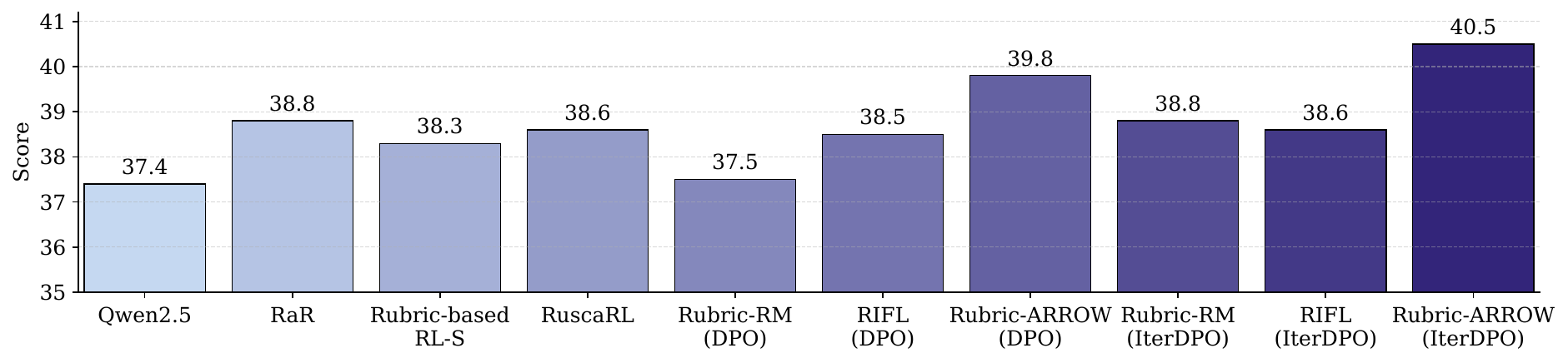}
  \caption{Trained-policy comparison on Creative Writing Benchmark v3. Baseline results (except Rubric-RM) are from RuscaRL~\cite{zhou2025breaking}; two baselines are omitted for readability (see Fig.~\ref{fig:cw_full} in Appendix~\ref{app:experiment} for the full version).\vspace{-1.5ex}}
  \label{fig:cw}
\end{figure}

\subsection{Performance of online RL}

\begin{table}[t!]
\centering
\caption{Comparison of online RL method with different alignment strategies applied to Qwen2.5-7B-Instruct on instruction following and preference alignment benchmarks. Best results are in \textbf{bold}.\vspace{-0.5ex}}
\label{tab:online_rl}
\resizebox{0.8\columnwidth}{!}{%
\begin{tabular}{l|cc|cc|cc|c}
\toprule
\multirow{2}{*}{Method} & \multicolumn{2}{c}{\bf IFEval (Prompt)} & \multicolumn{2}{c|}{\bf IFEval (Inst.)}   & \multicolumn{2}{c|}{\bf AlpacaEval}      & \multirow{2}{*}{\bf AVG} \\ \cmidrule(lr){2-5} \cmidrule(lr){6-7}
                       & \bf Loose &\bf  Strict &\bf  Loose  &\bf  Strict                                               &\bf  Vanilla & \bf Length                    &                      \\ \midrule
\multicolumn{1}{l|}{Qwen2.5-7B-Instruct}    & 75.0  & \multicolumn{1}{c|}{72.5}   & 81.8   & \multicolumn{1}{c|}{79.9}  & 33.5    & \multicolumn{1}{c|}{36.2} & 56.1                 \\
\multicolumn{1}{l|}{+GRPO (RM-R1)}          & 76.7  & \multicolumn{1}{c|}{73.6}   & 83.2   & \multicolumn{1}{c|}{80.2}  & \bf53.2    & \multicolumn{1}{c|}{42.7} & 63.2                 \\
\multicolumn{1}{l|}{+GRPO ({\name} w/o RL)}          &76.9   & \multicolumn{1}{c|}{74.5}   &83.5    & \multicolumn{1}{c|}{81.2}  & 51.4    & \multicolumn{1}{c|}{44.9} &   63.6               \\
\multicolumn{1}{l|}{+GRPO (RIFL)}          & 76.7  & \multicolumn{1}{c|}{73.8}   & 83.1   & \multicolumn{1}{c|}{80.6}  & 47.8    & \multicolumn{1}{c|}{42.4} & 61.9                 \\
\multicolumn{1}{l|}{+GRPO (Skywork)}          & 76.0  & \multicolumn{1}{c|}{72.3}   & 83.1   & \multicolumn{1}{c|}{80.3}  & 50.5    & \multicolumn{1}{c|}{45.6} & 63.0                 \\
\rowcolor{purple!10}
+GRPO ({\name})            & \bf 79.3  &\bf  76.2                        &\bf  84.9   &\bf  82.5                                            &  50.9    &\bf  49.3                      & \bf 65.4                 \\ \bottomrule
\end{tabular}%
}
\vspace{-2ex}
\end{table}

We evaluate {\name} in an online RL setting by directly optimizing Qwen2.5-7B-Instruct with GRPO using different reward models. As shown in Table~\ref{tab:online_rl}, GRPO via {\name} delivers the best overall average, improving from 56.1 for the base model to 65.4. It also outperforms reward baselines such as RM-R1, RIFL, and Skywork. Compared with the {\name} variant without RL, {\name} further improves the average score from 63.6 to 65.4. These results indicate that {\name} provides a stronger and more reliable reward signal for online policy optimization.

\subsection{Efficiency Comparison}

\begin{table}[t]
\centering
\small
\caption{Computing speed on 100 samples (vLLM). Results with ``$\star$'' were taken from \citep{xu2026alternating}{}.\vspace{-1ex}}
\label{tab:computing-speed}
\resizebox{0.6\columnwidth}{!}{%
\begin{tabular}{l c}
\toprule
 \textbf{Method}& \textbf{Compute}  \textbf{Time} \textbf{(s)} \\
\midrule
JudgeLRM-7B$^{\star}$ & 25.71 \\
RRM-7B$^{\star}$ & 203.40 \\
RM-R1-7B (Qwen-2.5-Inst)$^{\star}$ & 260.37 \\
RM-R1-7B (DeepSeek-Dist)$^{\star}$ & 170.76 \\
RM-R1-14B (Qwen-2.5-Inst)$^{\star}$ & 322.79 \\
RM-R1-14B (DeepSeek-Dist)$^{\star}$ & 382.02 \\
\textsc{Rubric-RM-8B}$^{\star}$ & 105.12 \\
\midrule
\rowcolor{purple!15}
\name-8B & 28.35 \\
\bottomrule
\end{tabular}
}
\vspace{-1ex}
\end{table}

We evaluate inference efficiency on 100 samples with vLLM (Table~\ref{tab:computing-speed}): \name-8B processes them in 28.35s, an order of magnitude faster than reasoning-based reward models (RRM-7B and RM-R1-7B/14B, all $>$170s) and substantially faster than Rubric-RM-8B (105.12s). Although JudgeLRM-7B is slightly faster (25.71s), it produces only a direct judgment without rubric-conditioned signals. \name{} thus preserves interpretability at a fraction of the reasoning-based inference cost.

\subsection{Case Study}

We qualitatively analyze failures of a baseline reward model on a challenging example. Table~\ref{tab:case-study} in Appendix~\ref{app:experiment} presents a comparison instance about the difference between birding and bird watching. RIFL produces several surface-level rubrics (e.g., conciseness, clear organization, neutral tone, and whether a direct comparison is made), and its judge marks the inaccurate Resp B as True on most of them; the response wins on aggregate even though the judge separately flags its factual inaccuracy. However, \name{} produces a focused set of hard-rule rubrics centered on accurately addressing the question, and the judge correctly marks Resp A as True and Resp B as False on the central direct-comparison rubric, yielding the right preference.

\section{Conclusion}
\label{sec:conclusion}

We propose {\name}, an alternating framework that trains a pointwise rubric reward model for non-verifiable domains. After a brief SFT warm-up, its alternating RL stage couples a probability-based scoring rule that reduces ties with phase-specific preference-based rewards and alternating GRPO to jointly train the rubric generator and the rubric-conditioned judge using only pairwise preference data, with no further frontier-LLM annotation. Theoretical analysis establishes preference consistency, variance reduction from opposite-side averaging, and phase-wise convergence guarantees for alternating GRPO. Empirically, {\name} attains the best average across reward-modeling benchmarks and yields the best downstream policies under both offline (DPO, IterDPO) and online (GRPO) training.

\bibliography{main}

\appendix
\section{Experiment}
\label{app:experiment}
\subsection{Implementation Details}

Tables~\ref{tab:hyperparameters-rubric} and~\ref{tab:hyperparameters-dpo} summarize the hyperparameters used for \name{} and policy model training, respectively. In rubric-based scoring, we assign a weight of 3 to each \textit{Hard Rule} and a weight of 1 to each \textit{Principle} when computing the final score. We train the GRPO models using the ms-swift library\footnote{\url{https://github.com/modelscope/ms-swift}}~\citep{zhao2025swift}, and conduct DPO and IterDPO training with LLaMA-Factory\footnote{\url{https://github.com/hiyouga/LLaMA-Factory}}~\citep{zheng2024llamafactory}. In total, \name{} is trained for three alternating RL iterations. The inference-time sampling parameters are reported in Table~\ref{tab:parameters-inference}. For baseline methods, we use the sampling configurations specified in their official implementations or original papers.

\begin{table}[h!]
\centering
\small
\caption{Hyper-parameters used in \name{} training.}
\resizebox{\columnwidth}{!}{%
\begin{tabular}{c|c|c|c|c|c}
\toprule
Module & Parameter & Value & Module & Parameter & Value \\
\midrule
\multirow{11}{*}{Rubric Generator}
& \#generations & 6
& \multirow{11}{*}{Judge}
& \#generations & 6 \\
& Cutoff Length & 512
& & Cutoff Length & 1024 \\
& Batch Size & 64
& & Batch Size & 32 \\
& Optimizer & AdamW
& & Optimizer & AdamW \\
& Learning Rate & 1e-6
& & Learning Rate & 1e-6 \\
& Temperature & 1.0
& & Temperature & 1.0 \\
& \#iterations & 2
& & \#iterations & 2 \\
& Epochs & 1
& & Epochs & 1 \\
& $\epsilon_{\mathrm{high}}$ & 0.28
& & $\epsilon_{\mathrm{high}}$ & 0.28 \\
& $\epsilon_{\mathrm{low}}$ & 0.2
& & $\epsilon_{\mathrm{low}}$ & 0.2 \\
& $\beta$ & 0.001
& & $\beta$ & 0.001 \\
\bottomrule
\end{tabular}%
}
\label{tab:hyperparameters-rubric}
\end{table}

\begin{table}[t]
\centering
\small
\caption{Hyper-parameters used in policy model training.}
\resizebox{\columnwidth}{!}{%
\begin{tabular}{c|c|c|c|c|c}
\toprule
Method & Parameter & Value & Method & Parameter & Value \\
\midrule
\multirow{11}{*}{DPO}
& Cutoff Length & 2048
& \multirow{11}{*}{GRPO}
& \#generations & 6 \\
& Batch Size & 64
& & Cutoff Length & 2048 \\
& Optimizer & AdamW
& & Batch Size & 64 \\
& Learning Rate & 8e-7
& & Optimizer & AdamW \\
& Epochs & 1
& & Learning Rate & 5e-7 \\
& beta & 0.1
& & Temperature & 1.0 \\
& SFT mixing weight & 0.3
& & \#iterations & 2 \\
& / & /
& & Epochs & 1 \\
& / & /
& & $\epsilon_{\mathrm{high}}$ & 0.28 \\
& / & /
& & $\epsilon_{\mathrm{low}}$ & 0.2 \\
& / & /
& & $\beta$ & 0.001 \\
\bottomrule
\end{tabular}%
}
\label{tab:hyperparameters-dpo}
\end{table}

\begin{table}[t]
\centering
\small
\caption{Sampling parameters used in \name{} inference.}
\resizebox{\columnwidth}{!}{%
\begin{tabular}{c|c|c|c|c|c}
\toprule
Module & Parameter & Value & Module & Parameter & Value \\
\midrule
\multirow{5}{*}{Rubric Generator}
& Maximum Tokens & 512
& \multirow{5}{*}{Judge}
& Maximum Tokens & 1024 \\
& Temperature & 0.0
& & Temperature & 1.0 \\
& Top-P & /
& & Top-P & 0.95 \\
& Top-K & /
& & Top-K & -1 \\
& Enable-thinking & False
& & Enable-thinking & False \\
\bottomrule
\end{tabular}%
}
\label{tab:parameters-inference}
\end{table}

\subsection{Evaluation Details}
\label{app:eval-details}

\paragraph{Pairwise construction for FollowBench and InfoBench.} For these two benchmarks, we convert the original single-response evaluation into a pairwise comparison setting: for each prompt, two responses are sampled from a fixed source model (Qwen-3-8B or Qwen-3-14B), and constraint violations are identified using each benchmark's official verifier.

\paragraph{Splits and metrics.} We follow each benchmark's official splits and scoring rules and report accuracy, win rate, or the benchmark-specific metric.

\subsection{Baselines}
\label{app:baselines}

\paragraph{Reward-model evaluation baselines.}
We compare \name\ with state-of-the-art white-box generative judge and reward models, including RM-R1~\citep{chen2025rm}, RRM~\citep{guo2026reward}, and JudgeLRM~\citep{chen2025judgelrm}. We also include black-box references for context: Claude-3.5-Sonnet$\footnote{\url{https://www.anthropic.com/news/claude-3-5-sonnet}}$, Gemini-2.5-Flash~\cite{comanici2025gemini}, and API-based rubric+judge pipelines (pairwise, pointwise, and direct-judge variants, with GPT-4.1-Mini$\footnote{\url{https://openai.com/index/gpt-4-1/}}$ as the rubric generator and Gemini-2.5-Flash-Lite as the judge). To isolate the effect of learned rubric-conditioned pointwise scoring, we further compare with prompted Qwen-3-8B rubric+judge pipelines (pairwise and pointwise) and with trained rubric-based baselines \textsc{Rubric-RM}~\citep{liu2025openrubrics}, {\name} w/o RL, and \textsc{RIFL}~\citep{he2025advancedif}.

\paragraph{Downstream policy training baselines.}
For downstream policy training, we provide training signals for Qwen2.5-7B-Instruct~\citep{qwen2.5} using different reward models. Skywork~\citep{liu2024skywork} and ArmoRM~\citep{wang2024interpretable} are scalar Bradley--Terry reward models trained on pairwise human preferences; for Skywork we use the \texttt{Skywork/Skywork-Reward-V2-Llama-3.1-8B} checkpoint. UltraFeedback~\citep{cui2023ultrafeedback} provides multi-aspect AI feedback annotations used as a preference dataset. RLCF~\citep{viswanathan2026checklists} is a checklist-based reward signal that converts explicit constraints into a scalar reward. We also compare against rubric-based reward models \textsc{Rubric-RM}, {\name} w/o RL, and \textsc{RIFL}.

\subsection{Length-Aware Rubric Reward}
\label{app:length-reward}

The rubric-generator reward \(R_{i,m}^{(r)}\) in Eq.~\eqref{eq:generator-rollout-reward} combines preference consistency with a length-aware bonus. Concretely, let
\begin{equation}
\mathcal C_i = \{m \in \{1,\ldots,n\} : q_i(s_{i,m}^1 - s_{i,m}^2) > 0\}
\end{equation}
denote the set of preference-consistent rubric rollouts for pair \(i\), and let
\begin{equation}
K_{i,m} = |\mathbf r_{i,m}|
\end{equation}
denote the number of rubric items in rollout \(m\). When \(\mathcal C_i \ne \emptyset\), we further define
\begin{equation}
K_i^{\min} = \min_{m \in \mathcal C_i} K_{i,m},
\qquad
\bar K_i = \frac{1}{|\mathcal C_i|} \sum_{m \in \mathcal C_i} K_{i,m}.
\end{equation}
The rubric-generator reward used in practice is
\begin{equation}
R_{i,m}^{(r)} =
\begin{cases}
1.1, & \text{if } m \in \mathcal C_i,\ K_{i,m} = K_i^{\min},\ \bar K_i \ge 5, \\
+1, & \text{if } m \in \mathcal C_i, \\
-1, & \text{otherwise.}
\end{cases}
\label{eq:generator-reward-detail}
\end{equation}
If \(\mathcal C_i = \emptyset\), all rubric rollouts receive reward \(-1\). This length-aware bonus rewards the shortest rubric set among those that correctly recover the human preference, while the condition \(\bar K_i \ge 5\) prevents the generator from being rewarded for producing overly short rubrics.

\subsection{Full Version of Policy Model Performance on IFBench and Creative Writing Benchmark v3}

Figs.~\ref{fig:ifbench_full} and \ref{fig:cw_full} report the full versions of the IFBench and Creative Writing v3 policy-model comparisons abbreviated in the main text, including the two baselines that were omitted from the main-body figures for readability.

\begin{figure}[h!]
  \centering
   \vspace{-0.5ex}
  \includegraphics[width=1\linewidth]{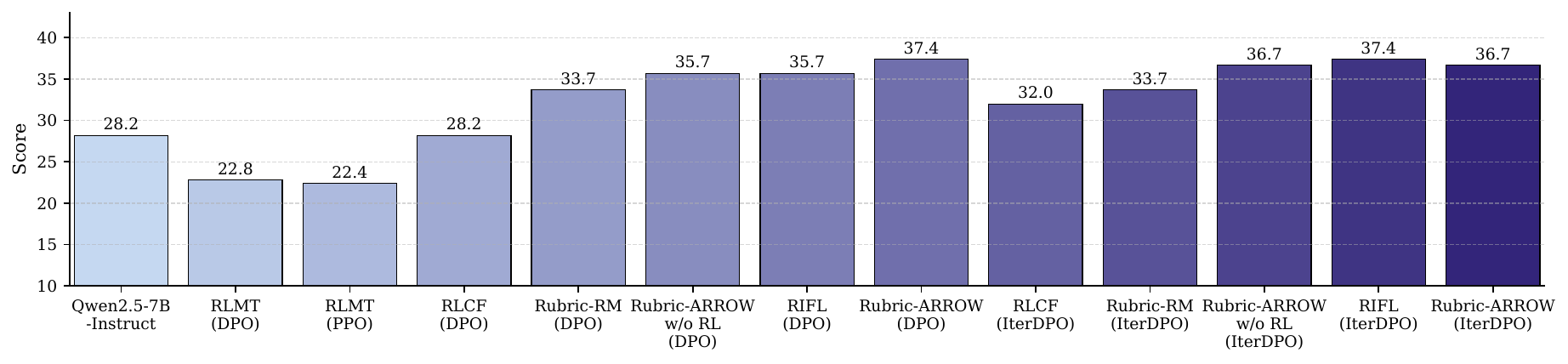}
\caption{Comparison of trained policy models on IFBench. Results of baselines except Rubric-RM (IterDPO) are from OpenRubrics~\cite{liu2025openrubrics}.\vspace{-1ex}}
  \label{fig:ifbench_full}
\end{figure}

\begin{figure}[h!]
  \centering
  \includegraphics[width=1\linewidth]{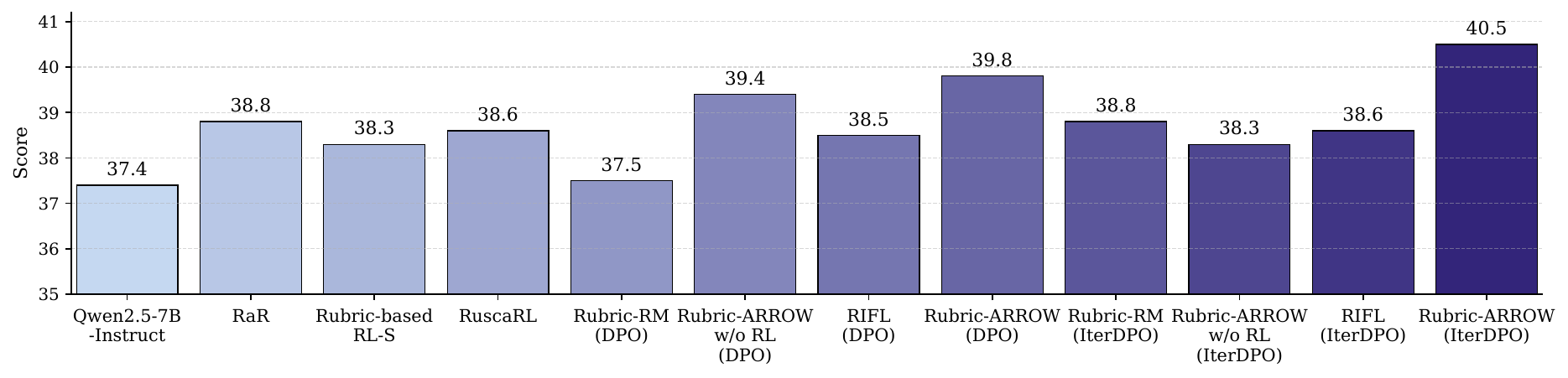}
  \caption{Comparison of trained policy models on Create Writing Benchmark v3. Results of baselines except Rubric-RM are from RuscaRL~\cite{zhou2025breaking}. \vspace{-1.5ex}}
  \label{fig:cw_full}
\end{figure}

\subsection{Case Study}

Table~\ref{tab:case-study} expands the case study from the main text, showing the prompt, the two candidate responses, and the full rubric and judge output for both RIFL and \name{}.

\begin{table*}[t]
\centering
\scriptsize
\caption{Case study with rubric-level error highlighting. The baseline marks several surface-level criteria as true for the incorrect response, while \name{} aligns each judge decision with the corresponding rubric and rejects the misleading comparison.}
\renewcommand{\arraystretch}{1.12}
\begin{tabular}{p{0.28\linewidth} p{0.48\linewidth} p{0.08\linewidth} p{0.08\linewidth}}
\toprule
\multicolumn{4}{l}{\textbf{Case Study on RewardBench Chat Hard}} \\
\midrule
\multicolumn{4}{p{0.96\linewidth}}{
\textbf{Prompt:} What is the difference between birding and bird watching?

\textbf{Resp A:} Birding is more active and committed, often involving documenting birds, while bird watching is a more casual or passive activity.

\textbf{Resp B:} Birding and bird watching both involve observing birds, but birding is done with the eyes, while bird watching is done with the ears.
} \\
\midrule

\multicolumn{4}{l}{\bf \textsc{RIFL}} \\
\midrule
\textbf{Rubric} & \textbf{Judge on Resp B} & \textbf{Resp A} & \textbf{Resp B} \\
\midrule

The response must directly compare birding and bird watching. [Hard Rule]
& ``The reply gives a direct comparison/contrast between birding and bird watching
(\textcolor{red}{\textbf{one uses eyes, the other uses ears}}), so it does compare the concepts,
\textcolor{red}{\textbf{albeit incorrectly}}.''
& True & \textcolor{red}{\textbf{True}} \\

The response should provide clear and accurate definitions for both concepts. [Principle]
& ``The assistant did not provide clear, accurate definitions for either concept; birding and bird watching are both defined incorrectly and imprecisely.''
& False & False \\

The response should identify key distinctions between birding and bird watching. [Principle]
& ``The stated distinctions are inaccurate and not the key distinctions usually cited, such as commitment, equipment, or identification/logging goals.''
& True & False \\

The response should use appropriate terminology. [Principle]
& ``Terminology is imprecise and inappropriate (`eyes' vs. `ears'); correct terms would include equipment, goals, habitat, or level of commitment.''
& False & False \\

The response should be clearly organized. [Principle]
& ``The response is organized simply, with a comparison followed by a contrast, and \textcolor{red}{\textbf{guides the reader clearly despite being content incorrect}}.''
& False & \textcolor{red}{\textbf{True}} \\

The response should be concise and avoid unnecessary information. [Principle]
& ``The answer is \textcolor{red}{\textbf{concise}} and contains no unnecessary information.''
& False & \textcolor{red}{\textbf{True}} \\

The response should maintain an objective and neutral tone. [Principle]
& ``Tone is \textcolor{red}{\textbf{objective and neutral}} throughout.''
& False & \textcolor{red}{\textbf{True}} \\

The response should be factually accurate. [Principle]
& ``The content is factually inaccurate regarding the distinctions between birding and bird watching.''
& False & False \\
\midrule
\multicolumn{4}{l}{\bf \textsc{\name{}}} \\
\midrule
\textbf{Rubric} & \textbf{Judge on Resp B} & \textbf{A} & \textbf{B} \\
\midrule

The response must directly address the user's question about the difference between birding and bird watching. [Hard Rule]
& ``The reply mentions both are about observing birds but does not accurately explain the real semantic distinction. The \textcolor{green!50!black}{\textbf{key difference in meaning is not correctly captured}}.''
& \textcolor{green!50!black}{\textbf{True}} & \textcolor{green!50!black}{\textbf{False}} \\

The response must provide clear and accurate definitions of both birding and bird watching. [Principle]
& ``The response lacks clear definitions for each term; it gives trivial distinctions rather than concise definitions of what birding and bird watching typically entail in practice.''
& False & False \\

The response must focus on the primary distinctions between the two terms, such as activity level, scope, goals, record-keeping, expertise, or equipment. [Principle]
& ``The stated methods (eyes vs. ears) do not reflect the primary objectives or typical scope differences. It fails to focus on the actual primary distinctions such as activity level, year-round involvement, and context.''
& False & False \\

\bottomrule
\end{tabular}
\label{tab:case-study}
\end{table*}

\section{Proofs of Theoretical Results}\label{app:proofs}
This appendix provides the technical assumptions and full proofs omitted from the main text.

\subsection{Auxiliary setup}
For the judge-side results, define
\begin{equation}
J_n
=
\frac{1}{2}\Prob\bigl(S^+>\bar{S}_n^-\bigr)
+
\frac{1}{2}\Prob\bigl(\bar{S}_n^+>S^-\bigr),
\label{eq:appendix-Jn}
\end{equation}
where \(\bar{S}_n^+\) and \(\bar{S}_n^-\) are the averages of \(n\) i.i.d. copies of \(S^+\) and \(S^-\), respectively.

For the optimization results, fix a phase \(b\in\{j,r\}\) and hold the other module \(\theta_{-b}\) constant throughout the phase. The phase-specific population objective is
\begin{equation}
J_b(\theta_b;\theta_{-b})
=
\begin{cases}
J_j(\theta_j;\theta_r), & b=j,\\
J_r(\theta_r;\theta_j), & b=r.
\end{cases}
\end{equation}
Let \(L_{G,b}(\theta_b;\theta_{-b})\) denote the one-step on-policy GRPO surrogate for phase \(b\), and let the GRPO iterates satisfy
\begin{equation}
\theta_b^{t+1} = \theta_b^t + \eta_t g_t^{(b)},
\qquad
t=0,1,\dots,T-1,
\label{eq:appendix-phase-update}
\end{equation}
where \(g_t^{(b)}\) is the stochastic gradient estimator for the active phase.

We use the following assumptions for Theorems~\ref{thm:surrogate-grpo} and \ref{thm:true-objective-neighborhood}.

\paragraph{A1. Phase-wise smoothness and upper boundedness.}
For each phase \(b\in\{j,r\}\), \(L_{G,b}(\theta_b;\theta_{-b})\) is \(L_b\)-smooth in \(\theta_b\) when \(\theta_{-b}\) is fixed, and \(L_{G,b}(\theta_b;\theta_{-b})\le L_{G,b}^{\star}\).

\paragraph{A2. Unbiased phase-wise gradient estimator.}
At each step,
\begin{equation}
\E[g_t^{(b)}\mid \theta_b^t]
=
\nabla_{\theta_b} L_{G,b}(\theta_b^t;\theta_{-b}).
\label{eq:appendix-unbiased}
\end{equation}

\paragraph{A3. Bounded conditional variance.}
There exists a constant \(\sigma_b^2\) such that
\begin{equation}
\E\!\left[\norm{g_t^{(b)}-\nabla_{\theta_b}L_{G,b}(\theta_b^t;\theta_{-b})}^2 \mid \theta_b^t\right]
\le
\frac{\sigma_b^2}{G_b},
\label{eq:appendix-var}
\end{equation}
where \(G_j=2n\) for the judge phase and \(G_r=n\) for the generator phase.

\paragraph{A4. Step size.}
We use
\begin{equation}
\eta_t=\frac{\eta_0}{\sqrt{T}},
\qquad
\eta_0\le \frac{1}{L_b}.
\label{eq:appendix-stepsize}
\end{equation}

\paragraph{A5. Phase-wise surrogate-to-true gradient gap.}
For each phase \(b\in\{j,r\}\), there exist nonnegative constants \(B_{\mathrm{grp}}^{(b)}\), \(B_{\mathrm{stale}}^{(b)}\), and \(B_{\mathrm{clip}}^{(b)}\) such that
\begin{align}
\notag&\sup_{\theta_b}
\E\norm{\nabla_{\theta_b}J_b(\theta_b;\theta_{-b})-\nabla_{\theta_b}L_{G,b}(\theta_b;\theta_{-b})}^2
\\&\le
\frac{B_{\mathrm{grp}}^{(b)}}{G_b} + B_{\mathrm{stale}}^{(b)} + B_{\mathrm{clip}}^{(b)}.
\label{eq:appendix-gradient-gap}
\end{align}

\subsection{Proof of Theorem~\ref{thm:pref-consistency}}
\begin{proof}
We first analyze the term \(\Prob(S^+>\bar{S}_n^-)\). Since
\begin{equation}
S^+-\bar{S}_n^- = (\mu^+-\mu^-) + \bigl(\xi^+ - \bar{\xi}_n^-\bigr)
= \Delta + Z_n^-,
\end{equation}
where \(\bar{\xi}_n^-:=\frac{1}{n}\sum_{m=1}^{n}\xi_m^-\) and \(Z_n^-:=\xi^+-\bar{\xi}_n^-\), we have
\begin{equation}
\Prob(S^+>\bar{S}_n^-)
=
\Prob\bigl(Z_n^->-\Delta\bigr).
\label{eq:appendix-term1}
\end{equation}
Because \(\xi^+\) and each \(\xi_m^-\) are independent and symmetric about \(0\), the average \(\bar{\xi}_n^-\) is symmetric about \(0\), and therefore \(Z_n^- = \xi^+ - \bar{\xi}_n^-\) is also symmetric about \(0\). Since \(Z_n^-\) is continuous, its median is \(0\), which implies
\begin{equation}
\Prob(Z_n^->0)=\frac{1}{2}.
\end{equation}
Moreover, the map \(c\mapsto \Prob(Z_n^->c)\) is strictly decreasing in \(c\). Therefore,
\begin{equation}
\Prob\bigl(Z_n^->-\Delta\bigr)
\begin{cases}
>\frac{1}{2}, & \Delta>0,\\
=\frac{1}{2}, & \Delta=0,\\
<\frac{1}{2}, & \Delta<0.
\end{cases}
\label{eq:appendix-term1-sign}
\end{equation}

The second term is analogous. Writing
\begin{equation}
\bar{S}_n^+ - S^- = \Delta + Z_n^+,
\qquad
Z_n^+ := \bar{\xi}_n^+ - \xi^-,
\end{equation}
we again obtain that \(Z_n^+\) is continuous and symmetric about \(0\), hence
\begin{equation}
\Prob(\bar{S}_n^+>S^-)
\begin{cases}
>\frac{1}{2}, & \Delta>0,\\
=\frac{1}{2}, & \Delta=0,\\
<\frac{1}{2}, & \Delta<0.
\end{cases}
\label{eq:appendix-term2-sign}
\end{equation}
Averaging Equations~\eqref{eq:appendix-term1-sign} and \eqref{eq:appendix-term2-sign} inside Equation~\eqref{eq:appendix-Jn} yields
\begin{equation}
J_n
\begin{cases}
>\frac{1}{2}, & \Delta>0,\\
=\frac{1}{2}, & \Delta=0,\\
<\frac{1}{2}, & \Delta<0.
\end{cases}
\end{equation}
This proves the claim.
\end{proof}

\subsection{Proof of Theorem~\ref{thm:variance-reduction}}
\begin{proof}
Because Gaussian distributions are closed under averaging and subtraction, we have
\begin{equation}
S^+ - \bar{S}_n^- = \Delta + \bigl(\xi^+ - \bar{\xi}_n^-\bigr),
\end{equation}
with
\begin{equation}
\xi^+ - \bar{\xi}_n^- \sim \mathcal{N}\!\left(0,\sigma_+^2+\frac{\sigma_-^2}{n}\right).
\end{equation}
Therefore,
\begin{equation}
\Prob(S^+>\bar{S}_n^-)
=
\Phi\!\left(\frac{\Delta}{\sqrt{\sigma_+^2+\sigma_-^2/n}}\right).
\label{eq:appendix-first-branch-gaussian}
\end{equation}
Similarly,
\begin{equation}
\bar{S}_n^+ - S^- = \Delta + \bigl(\bar{\xi}_n^+ - \xi^-\bigr),
\end{equation}
with
\begin{equation}
\bar{\xi}_n^+ - \xi^- \sim \mathcal{N}\!\left(0,\frac{\sigma_+^2}{n}+\sigma_-^2\right),
\end{equation}
which implies
\begin{equation}
\Prob(\bar{S}_n^+>S^-)
=
\Phi\!\left(\frac{\Delta}{\sqrt{\sigma_-^2+\sigma_+^2/n}}\right).
\label{eq:appendix-second-branch-gaussian}
\end{equation}
Combining the two branches gives
\begin{equation}
\begin{aligned}
J_n
= \frac{1}{2}\Bigg[
&\Phi\!\left(\frac{\Delta}{\sqrt{\sigma_+^2+\sigma_-^2/n}}\right) \\
&+
\Phi\!\left(\frac{\Delta}{\sqrt{\sigma_-^2+\sigma_+^2/n}}\right)
\Bigg].
\end{aligned}
\label{eq:appendix-hetero-jn}
\end{equation}

For the baseline single pairwise comparison, note that
\begin{equation}
S^+-S^- \sim \mathcal{N}\!\left(\Delta,\sigma_+^2+\sigma_-^2\right),
\end{equation}
so
\begin{equation}
J_{\mathrm{pair}} = \Phi\!\left(\frac{\Delta}{\sqrt{\sigma_+^2+\sigma_-^2}}\right).
\end{equation}
If \(n>1\), then
\begin{equation}
\sigma_+^2+\frac{\sigma_-^2}{n} < \sigma_+^2+\sigma_-^2,
\qquad
\sigma_-^2+\frac{\sigma_+^2}{n} < \sigma_-^2+\sigma_+^2.
\end{equation}
Since \(\Phi\) is strictly increasing and \(\Delta>0\), each argument in Equation~\eqref{eq:appendix-hetero-jn} is strictly larger than the argument defining \(J_{\mathrm{pair}}\), hence \(J_n>J_{\mathrm{pair}}\).

Finally, if \(\sigma_+=\sigma_-:=\sigma\), then both terms in Equation~\eqref{eq:appendix-hetero-jn} coincide and equal
\begin{equation}
\Phi\!\left(\frac{\Delta}{\sigma\sqrt{1+1/n}}\right),
\end{equation}
which is strictly increasing in \(n\) whenever \(\Delta>0\).
\end{proof}

\subsection{Proof of Theorem~\ref{thm:surrogate-grpo}}
\begin{proof}
Fix a phase \(b\in\{j,r\}\) and suppress the frozen variable \(\theta_{-b}\) for readability. Since \(L_{G,b}\) is \(L_b\)-smooth in \(\theta_b\), we have
\begin{equation}
\begin{aligned}
L_{G,b}(\theta_b^{t+1})
&\ge L_{G,b}(\theta_b^{t})
+ \eta_t\langle \nabla_{\theta_b}L_{G,b}(\theta_b^{t}), g_t^{(b)}\rangle \\
&\quad - \frac{L_b\eta_t^2}{2}\norm{g_t^{(b)}}^2.
\end{aligned}
\label{eq:appendix-phase-smooth}
\end{equation}
Conditioning on \(\theta_b^t\) and using Equation~\eqref{eq:appendix-unbiased},
\begin{equation}
\E\bigl[\langle \nabla_{\theta_b}L_{G,b}(\theta_b^{t}), g_t^{(b)}\rangle \mid \theta_b^t\bigr]
=
\norm{\nabla_{\theta_b}L_{G,b}(\theta_b^{t})}^2.
\end{equation}
Moreover, by Equation~\eqref{eq:appendix-var},
\begin{equation}
\begin{aligned}
\E\bigl[\norm{g_t^{(b)}}^2 \mid \theta_b^t\bigr]
&=
\norm{\nabla_{\theta_b}L_{G,b}(\theta_b^{t})}^2 \\
 + &\E\bigl[\norm{g_t^{(b)}-\nabla_{\theta_b}L_{G,b}(\theta_b^{t})}^2 \mid \theta_b^t\bigr]\\
&\le \norm{\nabla_{\theta_b}L_{G,b}(\theta_b^{t})}^2+\frac{\sigma_b^2}{G_b}.
\end{aligned}
\end{equation}
Substituting these identities into Equation~\eqref{eq:appendix-phase-smooth} gives
\begin{equation}
\begin{aligned}
&\E\bigl[L_{G,b}(\theta_b^{t+1}) \mid \theta_b^t\bigr] \\
&\quad\ge L_{G,b}(\theta_b^{t}) \\
&\qquad + \eta_t\!\left(1-\tfrac{L_b\eta_t}{2}\right)\norm{\nabla_{\theta_b}L_{G,b}(\theta_b^{t})}^2 \\
&\qquad - \frac{L_b\eta_t^2\sigma_b^2}{2G_b}.
\end{aligned}
\end{equation}
Because \(\eta_t\le 1/L_b\), we have \(1-\frac{L_b\eta_t}{2}\ge \frac{1}{2}\). Taking full expectation and summing over \(t=0,\ldots,T-1\),
\begin{equation}
\begin{aligned}
\frac{1}{2}\sum_{t=0}^{T-1}\eta_t \,
\E\norm{\nabla_{\theta_b}L_{G,b}(\theta_b^{t})}^2
&\le
L_{G,b}^{\star} - L_{G,b}(\theta_b^{0}) \\
&\quad+
\frac{L_b\sigma_b^2}{2G_b}
\sum_{t=0}^{T-1}\eta_t^2 .
\end{aligned}
\label{eq:appendix-phase-telescope}
\end{equation}
Under Equation~\eqref{eq:appendix-stepsize},
\begin{equation}
\sum_{t=0}^{T-1}\eta_t = \eta_0\sqrt{T},
\qquad
\sum_{t=0}^{T-1}\eta_t^2 = \eta_0^2.
\end{equation}
Dividing Equation~\eqref{eq:appendix-phase-telescope} by \(\frac{1}{2}\eta_0\sqrt{T}\) yields
\begin{equation}
\begin{aligned}
\frac{1}{T}\sum_{t=0}^{T-1}
\E\norm{\nabla_{\theta_b}L_{G,b}(\theta_b^{t})}^2
&\le
\frac{2\bigl(L_{G,b}^{\star}-L_{G,b}(\theta_b^{0})\bigr)}
{\eta_0\sqrt{T}} \\
&\quad+
\frac{L_b\eta_0\sigma_b^2}{G_b\sqrt{T}} .
\end{aligned}
\label{eq:appendix-phase-surrogate-bound}
\end{equation}
which implies the stated \(O(T^{-1/2})\) rate.
\end{proof}

\subsection{Proof of Theorem~\ref{thm:true-objective-neighborhood}}
\begin{proof}
At each step \(t\), decompose the phase-wise true gradient as
\begin{equation}
\begin{aligned}
\nabla_{\theta_b}J_b(\theta_b^t)
&=
\nabla_{\theta_b}L_{G,b}(\theta_b^t) \\
&\quad+
\Bigl(
\nabla_{\theta_b}J_b(\theta_b^t)
-
\nabla_{\theta_b}L_{G,b}(\theta_b^t)
\Bigr).
\end{aligned}
\end{equation}
Using \(\norm{u+v}^2\le 2\norm{u}^2 + 2\norm{v}^2\), we obtain
\begin{equation}
\begin{aligned}
\norm{\nabla_{\theta_b}J_b(\theta_b^t)}^2
&\le
2\norm{\nabla_{\theta_b}L_{G,b}(\theta_b^t)}^2 \\
&\quad+
2\norm{
\nabla_{\theta_b}J_b(\theta_b^t)
-
\nabla_{\theta_b}L_{G,b}(\theta_b^t)
}^2 .
\end{aligned}
\label{eq:appendix-phase-true-vs-surrogate}
\end{equation}
Take expectations and average over \(t=0,\ldots,T-1\):
\begin{equation}
\begin{aligned}
\frac{1}{T}\sum_{t=0}^{T-1}
\E&\norm{\nabla_{\theta_b}J_b(\theta_b^t)}^2
\le
\frac{2}{T}\sum_{t=0}^{T-1}
\E\norm{\nabla_{\theta_b}L_{G,b}(\theta_b^t)}^2 \\
&+
\frac{2}{T}\sum_{t=0}^{T-1}
\E\norm{
\nabla_{\theta_b}J_b(\theta_b^t)
-
\nabla_{\theta_b}L_{G,b}(\theta_b^t)
}^2 .
\end{aligned}
\label{eq:appendix-phase-average-split}
\end{equation}
The first term is controlled by Equation~\eqref{eq:appendix-phase-surrogate-bound}. For the second term, apply the gradient-gap assumption in Equation~\eqref{eq:appendix-gradient-gap} uniformly over the iterates:
\begin{equation}
\begin{aligned}
&\frac{2}{T}\sum_{t=0}^{T-1}
\E\norm{
\nabla_{\theta_b}J_b(\theta_b^t)
-
\nabla_{\theta_b}L_{G,b}(\theta_b^t)
}^2
\\&\le
\frac{2B_{\mathrm{grp}}^{(b)}}{G_b}
+
2B_{\mathrm{stale}}^{(b)}
+
2B_{\mathrm{clip}}^{(b)} .
\end{aligned}
\label{eq:appendix-phase-gap-bound}
\end{equation}
Substituting Equations~\eqref{eq:appendix-phase-surrogate-bound} and \eqref{eq:appendix-phase-gap-bound} into Equation~\eqref{eq:appendix-phase-average-split} yields
\begin{equation}
\begin{aligned}
&\frac{1}{T}\sum_{t=0}^{T-1}\E\norm{\nabla_{\theta_b}J_b(\theta_b^t)}^2
\\&\le
\frac{4\bigl(L_{G,b}^{\star}-L_{G,b}(\theta_b^{0})\bigr)}{\eta_0\sqrt{T}}
+ \frac{2L_b\eta_0\sigma_b^2}{G_b\sqrt{T}}\\
&+ \frac{2B_{\mathrm{grp}}^{(b)}}{G_b}
+ 2B_{\mathrm{stale}}^{(b)}
+ 2B_{\mathrm{clip}}^{(b)}.
\end{aligned}
\label{eq:appendix-phase-true-bound}
\end{equation}
This proves the claimed \(O(T^{-1/2})+O(G_b^{-1})\) stationary-neighborhood bound for the phase-specific true objective. Since \(G_j=2n\) and \(G_r=n\), the rollout-dependent term is \(O(n^{-1})\) in both phases.
\end{proof}

\section{Prompts}

We report the prompts used in our experiments in this section. For baseline methods, we follow the prompts provided by their official implementations and original papers.

\begin{figure*}[t]
\centering
\label{fig:generate-prompt}
\begin{tcolorbox}[
    enhanced,
    width=0.96\textwidth,
    colback=promptbg,
    colframe=promptpink,
    coltitle=white,
    title={Prompt for Rubric Generation (\name)},
    fonttitle=\bfseries\large,
    boxrule=1.1pt,
    arc=3pt,
    outer arc=3pt,
    left=8pt,
    right=8pt,
    top=8pt,
    bottom=8pt,
    toptitle=4pt,
    bottomtitle=4pt,
    lefttitle=8pt,
    righttitle=8pt,
]
\begin{Verbatim}[
    fontsize=\tiny,
    breaklines=true,
    breakanywhere=true,
    breaksymbolleft={},
    breaksymbolright={}
]
    "Your task is to extract a set of rubric-style instructions from a user's request.\n"
    "These rubrics will be used as evaluation criteria to check if a response fully meets the request.\n"
    "Every rubric item must be a universal principle. If any rubric still contains topic-specific references (e.g., names, places, myths, numbers, historical facts), it is automatically invalid.\n"
    "\n"
    "- **Two Distinct Categories:**\n"
    "  - [Hard Rule]: Derived strictly from explicit requirements stated in the <request> (format, length, structure, forbidden/required elements, etc.).\n"
    "  - [Principle]: Derived by abstracting any concrete cues into domain-agnostic quality criteria (e.g., clarity, correctness, sound reasoning, pedagogy).\n"
    "\n"
    "- **Comprehensiveness:**\n"
    "  The rubric must cover all critical aspects implied by the request and examples, including explicit requirements and implicit quality standards.\n"
    "\n"
    "- **Conciseness & Uniqueness:**\n"
    "  Each rubric must capture a distinct evaluation criterion. Overlapping or redundant criteria must be merged into a single rubric. Wording must be precise and free of repetition.\n"
    "\n"
    "- **Format Requirements:**\n"
    "  - Use a numbered list.\n"
    "  - Each item starts with \"The response\" phrased in third person.\n"
    "  - Append [Hard Rule] or [Principle] at the end of each item.\n"
    "  - Do not include reasoning, explanations, or examples in the final output—only the rubrics.\n"
    "\n"
    "Here is the request:\n"
    "{prompt}\n"
    "\n"
    "Please generate the rubrics for the above request."
\end{Verbatim}
\end{tcolorbox}
\end{figure*}

\begin{figure*}[t]
\centering
\begin{tcolorbox}[
    enhanced,
    width=0.96\textwidth,
    colback=promptbg,
    colframe=promptpink,
    coltitle=white,
    title={Prompt for Judge Generation (\name)},
    fonttitle=\bfseries\large,
    boxrule=1.1pt,
    arc=3pt,
    outer arc=3pt,
    left=8pt,
    right=8pt,
    top=8pt,
    bottom=8pt,
    toptitle=4pt,
    bottomtitle=4pt,
    lefttitle=8pt,
    righttitle=8pt,
]
\begin{Verbatim}[
    fontsize=\tiny,
    breaklines=true,
    breakanywhere=true,
    breaksymbolleft={},
    breaksymbolright={}
]
Your job is to look at a conversation and a set of rubric items, and score the last turn (i.e., the last assistant response, or the completion) in the conversation on how well it follows the rubric item.

# Conversation
<<conversation>>

# Rubric item
<<rubric_item>>

# Instructions
Return a json object. For each rubric item i (starting from 1), keys must be exactly "explanation_i" and "criteria_met_i" for each i and it includes two top-level fields in the JSON object:
- The "explanation_i" field should be a string explaining why the response does or does not meet the criteria of the rubric item.
- The "criteria_met_i" field should be a boolean indicating (true/false) whether the response meets the criteria of the rubric item. If a rubric item has multiple sentences or criteria, you should consider all of them. If any of the criteria is not met, the answer should be false. Only return true is all of the criteria are met.
- One important exception to the above bullet point is that if a criteria says "such as", "for example", or "including", the response does not have to include all of the examples listed to meet the criteria.

# Final Output Format (a single JSON object, not an array)
{
  "explanation_1": "...",
  "criteria_met_1": true/false,
  "explanation_2": "...",
  "criteria_met_2": true/false,
  ... repeat this pattern for every rubric item i in order (i = 1, 2, 3, ...)
}

# Final instruction
Return just the json object. Do not include any other text in the response.
\end{Verbatim}
\end{tcolorbox}
\end{figure*}

\appendix
\end{document}